%% file: main_arxiv.tex
\documentclass{article}
\pdfoutput=1

\include{header}
\usepackage{ss}
\include{cymacros_arxiv}

\definecolor{LightGray}{gray}{0.9}

\allowdisplaybreaks
\usepackage[flushleft]{threeparttable} % http://ctan.org/pkg/threeparttable

\usepackage{enumitem}
\allowdisplaybreaks

\usepackage{subcaption}

\newcommand{\singshuf}{\textsc{SingleShuffle}}
\newcommand{\randshuf}{\textsc{RandomShuffle}}
\newcommand{\sgd}{\textsc{Sgd}}
\newcommand{\ssmean}{\mW_{\rm SS}}
\newcommand{\rsmean}{\mW_{\rm RS}}
\newcommand{\gdmean}{\mW_{\rm GD}}
\newcommand{\I}{\mathrm{i}}
\renewcommand{\Re}{{\rm Re}}
\newcommand{\Ewo}{\E_{\rm wo}}
\newcommand{\Ewr}{\E_{\rm wr}}

\title{Can Single-Shuffle SGD be Better than\\Reshuffling SGD and GD?}
\author{\name Chulhee Yun \email{chulheey@mit.edu}\\
  \name Suvrit Sra \email{suvrit@mit.edu}\\
  \name Ali Jadbabaie \email{jadbabai@mit.edu}\\
  \addr{Massachusetts Institute of Technology, Cambridge, MA, USA 02139}
}

\begin{document}

\maketitle

\input{0_abstract_arxiv}
\input{1_intro_arxiv}
\input{2_background_arxiv}

\input{3_conjecture_arxiv}
\input{4_theorems_arxiv}

\input{5_pfsketch_arxiv}

\section*{Acknowledgements}
CY acknowledges Korea Foundation for Advanced Studies, NSF CAREER grant 1846088, and ONR grant N00014-20-1-2394 for financial support.
SS acknowledges support from NSF BIGDATA grant 1741341, NSF CAREER grant 1846088, an MIT RSC award, and an Amazon Research Award.
AJ acknowledges support from ONR grant N00014-20-1-2394, and from MIT-IBM Watson AI lab.

\bibliographystyle{plainnat}
\bibliography{cite}

\appendix
\newpage

\input{A_counterex_arxiv}
\input{B_otherver_arxiv}

\input{C_pflemmas_arxiv}

\end{document}

%% file: header.tex
\usepackage[utf8]{inputenc}
\usepackage[T1]{fontenc}

\usepackage[verbose=true,
letterpaper,
textheight=8.7in,
textwidth=6.4in,
margin=1in,
headheight=12pt,
headsep=25pt,
footskip=30pt]{geometry}

\usepackage{microtype}
\usepackage{mathtools}
\usepackage{booktabs}
\usepackage[numbers,sort&compress]{natbib}
\usepackage{graphicx}
\usepackage{amsmath,amssymb,amsfonts,amsxtra,bm}
\usepackage{amsthm}
\usepackage{xspace}
\usepackage{enumerate}
\usepackage{hyperref}
\usepackage{url}

\usepackage{colortbl}

\usepackage{makecell,multirow}       % professional-quality tables
\usepackage{nicefrac}       % compact symbols for 1/2, etc.
\usepackage{thmtools}  
\usepackage{xspace}
\usepackage{footnote, tablefootnote}
\usepackage{float}
\floatstyle{plaintop}
\restylefloat{table}
\usepackage{epstopdf}
\usepackage{latexsym}
\usepackage{graphicx}

%% file: cymacros_arxiv.tex
\usepackage{enumerate}
\usepackage{amssymb}
\usepackage{amsbsy}
\usepackage{amsmath}
\usepackage{amsfonts}
\usepackage{thmtools}
\usepackage{mathtools}

\usepackage{color}

\newcommand{\mc}[1]{\mathcal{#1}}
\newcommand{\mbf}[1]{\mathbf{#1}}

\newcommand{\norm}[1]{\left\|{#1}\right\|} % A norm with 1 argument
\newcommand{\lfro}[1]{\left\|{#1}\right\|_{\rm F}} % Frobenius norm
\newcommand{\norms}[1]{\|{#1}\|} % A norm with 1 argument and normal (small)
\newcommand{\spann}{\mathop{\rm span}}

\newcommand{\ones}{{\mathbf{1}}}
\newcommand{\zeros}{{\mathbf{0}}}

\newcommand{\defeq}{:=}

\newcommand{\wt}[1]{\widetilde{#1}} % Wide tilde

\newcommand{\half}{\tfrac{1}{2}}
\newcommand{\bhalf}{\frac{1}{2}}
\newcommand{\indic}{\mbf{1}} % Indicator function

\newcommand{\reals}{\mathbb{R}} % Real number symbol
\newcommand{\C}{\mathbb{C}} % Complex numbers
\makeatletter
\long\def\@makecaption#1#2{
  \vskip 0.8ex
  \setbox\@tempboxa\hbox{\small {\bf #1:} #2}
  \parindent 1.5em  %% How can we use the global value of this???
  \dimen0=\hsize
  \advance\dimen0 by -3em
  \ifdim \wd\@tempboxa >\dimen0
  \hbox to \hsize{
    \parindent 0em
    \hfil 
    \parbox{\dimen0}{\def\baselinestretch{0.96}\small
      {\bf #1.} #2
    } 
    \hfil}
  \else \hbox to \hsize{\hfil \box\@tempboxa \hfil}
  \fi
}
\makeatother

\providecommand{\minimize}{\mathop{\rm minimize}}

\newcommand{\leqwtxt}[1]{\stackrel{{#1}}{\leq}}

\usepackage{bm}

\def\1{\bm{1}}

\def\vb{{\bm{b}}}

\def\vu{{\bm{u}}}
\def\vv{{\bm{v}}}

\def\vx{{\bm{x}}}
\def\vy{{\bm{y}}}
\def\vz{{\bm{z}}}

\def\mA{{\bm{A}}}
\def\mB{{\bm{B}}}
\def\mC{{\bm{C}}}

\def\mI{{\bm{I}}}

\def\mM{{\bm{M}}}

\def\mP{{\bm{P}}}
\def\mQ{{\bm{Q}}}

\def\mS{{\bm{S}}}
\def\mT{{\bm{T}}}
\def\mU{{\bm{U}}}
\def\mV{{\bm{V}}}
\def\mW{{\bm{W}}}
\def\mX{{\bm{X}}}
\def\mY{{\bm{Y}}}
\def\mZ{{\bm{Z}}}

\def\mLambda{{\bm{\Lambda}}}

\DeclareMathAlphabet{\mathsfit}{\encodingdefault}{\sfdefault}{m}{sl}
\SetMathAlphabet{\mathsfit}{bold}{\encodingdefault}{\sfdefault}{bx}{n}

\def\sD{{\mathbb{D}}}
\def\sS{{\mathbb{S}}}

\newcommand{\E}{\mathbb{E}}

\newcommand{\R}{\mathbb{R}}

%% file: 0_abstract_arxiv.tex
\begin{abstract}
We propose matrix norm inequalities that extend the \citet{recht2012toward} conjecture on a noncommutative AM-GM inequality by supplementing it with another inequality that accounts for \emph{single-shuffle}, which is a widely used without-replacement sampling scheme that shuffles only once in the beginning and is overlooked in the Recht-R\'e conjecture.
Instead of general positive semidefinite matrices, we restrict our attention to positive definite matrices with small enough condition numbers, which are more relevant to matrices that arise in the analysis of SGD. For such matrices, we conjecture that the means of matrix products corresponding to with- and without-replacement variants of SGD satisfy a series of spectral norm inequalities that can be summarized as: ``\emph{single-shuffle SGD converges faster than random-reshuffle SGD, which is in turn faster than with-replacement SGD.}'' We present theorems that support our conjecture by proving several special cases.
\end{abstract}

%%% Local Variables:
%%% mode: latex
%%% TeX-master: "main"
%%% End:

%% file: 1_intro_arxiv.tex
\section{Introduction}
\label{sec:intro}
\lettrine[lines=3]{\color{BrickRed}F}{\rm{}inite-sum} optimization problems are ubiquitous in machine learning. Typically, these problems are high-dimensional and the finite-sum is comprised of a large number of component functions, making the evaluation of exact gradients prohibitively expensive. Hence, randomized methods such as stochastic gradient descent (SGD) and its variants have become indispensable to modern machine learning. At each iteration, these methods evaluate the gradient of one component function sampled from the entire set of components and use it as a noisy estimate of the full gradient.

Depending on how the components are chosen, SGD-based methods can broadly be classified into two categories: \emph{with-replacement} and \emph{without-replacement}. In many theoretical studies the indices of component functions are assumed to be chosen with replacement, making the choice at each iteration independent of other iterations. In contrast, the vast majority of practical implementations use without-replacement sampling, where all the indices are randomly shuffled and are then visited exactly once per \emph{epoch} (i.e., one pass through all the components). Since the algorithms operate epochwise, there are two popular variants of shuffling schemes: one that \emph{reshuffles} the components at every epoch and another that \emph{shuffles only once} at the beginning and reuses that order every epoch. %and iterates through that order multiple times.

Practitioners opt for without-replacement sampling schemes not only because they are easier to implement, but also because they often result in faster convergence \citep{bottou2009curiously}.
However, analyzing algorithms based on without-replacement sampling is considerably more difficult than analyzing their with-replacement counterparts, because the component chosen at each iteration is dependent on the previous iterates of an epoch. Even for SGD, tight convergence bounds for its without-replacement variants have become known only recently \citep{nagaraj2019sgd, safran2019good, rajput2020closing, ahn2020sgd, mishchenko2020random}.

In 2012, \citet{recht2012toward} proposed a conjecture that the mean of without-replacement products of positive semidefinite (PSD) matrices has spectral norm no larger than the mean of their with-replacement products. Formally, given $n$ PSD matrices $\mA_1, \dots, \mA_n$, they conjectured:
\begin{equation}
    \norm{\frac{1}{n!} \sum\nolimits_{\sigma \in \mc S_n} \prod\nolimits_{i=1}^n \mA_{\sigma(i)}} 
    \leq
    \norm{\frac{1}{n} \sum\nolimits_{i=1}^n \mA_i}^n,
\label{eq:recht-re-conj}
\end{equation}
where $\mc S_n$ is the set of all permutations on $\{1,\dots,n\}$.
This so-called \emph{matrix AM-GM inequality}\footnote{Developing a matrix counterpart of the well-known scalar arithmetic-geometric means inequality has been a longstanding research topic. Different geometric means and their corresponding AM-GM inequalities have been proposed and studied in the literature \citep{bhatia1993more,horn1995norm,bhatia2000notes,ando2004geometric,bhatia2006riemannian,bini2010effective,bhatia2012monotonicity}.} conjecture has drawn much attention because it can help explain why without-replacement algorithms may converge faster than with-replacement ones. Inequality~\eqref{eq:recht-re-conj} is true for $n=2$ \citep{recht2012toward}, and $n=3$ \citep{zhang2018note,lai2020recht};
however, it was recently proven to be false for $n=5$~\citep{lai2020recht} and explicit counterexamples were found for $n \geq 5$~\citep{de2020random}.

Although this conjecture-and-disproof story seems inescapable, we point out a couple of important facts about the conjectured inequality~\eqref{eq:recht-re-conj} and its disproof, reigniting new hope.
\begin{itemize}[leftmargin=15pt,itemsep=-1pt]
    \item The counterexamples are rank-deficient PSD matrices, and adding a positive multiple of the identity matrix to each matrix makes them satisfy \eqref{eq:recht-re-conj} (see Section~\ref{sec:breaking-disproofs}). This leaves open the possibility that the conjecture may still be true for positive definite matrices with sufficiently small condition numbers. In fact, matrices that arise in the analyses of without-replacement methods are often those with small condition numbers (see Section~\ref{sec:moti-sgd} for details).
    \item The conjecture~\eqref{eq:recht-re-conj} only implies faster convergence of the without-replacement scheme where the indices are randomly reshuffled at every epoch. It does not provide any useful insights towards the analysis of \emph{single-shuffle}, an equally (if not more) popular scheme that shuffles indices only once and reuses that order for the remaining epochs.
\end{itemize}

\paragraph{Our contributions \& paper organization.} 
We propose a new conjecture (Conjecture~\ref{conj:main}) that extends and also refines the previously conjectured matrix AM-GM inequality~\eqref{eq:recht-re-conj}.
Our new conjecture is motivated by a rather surprising empirical observation from linear regression (see Figure~\ref{fig:linearreg}(a); details are in Section~\ref{sec:empirical}), namely that SGD with single shuffling converges consistently faster than SGD with reshuffling, which in turn outperforms with-replacement SGD and gradient descent (GD).

Our conjecture states that for any $n \geq 2$ and $K \geq 1$, positive definite (PD) matrices $\mA_1,\dots,\mA_n$ with ``small enough'' condition numbers satisfy the following inequalities:
\begin{equation}
    \norm{\frac{1}{n!} \sum\nolimits_{\sigma \in \mc S_n} \left(\prod\nolimits_{i=1}^n \mA_{\sigma(i)} \right)^K}
    \leq
    \norm{\frac{1}{n!} \sum\nolimits_{\sigma \in \mc S_n} \prod\nolimits_{i=1}^n \mA_{\sigma(i)}}^K
    \leq
    \norm{\frac{1}{n} \sum\nolimits_{i=1}^n \mA_i}^{nK}.
\label{eq:our-conj}
\end{equation}
These two spectral norm inequalities imply that, at least in some special cases,
\begin{quote}
\begin{center}
  \emph{Shuffling only once is faster than randomly reshuffling at every epoch,\\and random reshuffling is faster than with-replacement sampling.}  
\end{center}
\end{quote}
%The rest of the paper is organized as follows.
In Section~\ref{sec:background}, we introduce with- and without-replacement SGD and connect their convergence analysis to our conjecture~\eqref{eq:our-conj}. We also discuss motivations for our extended conjecture~\eqref{eq:our-conj} via our empirical observations and disproofs of \eqref{eq:recht-re-conj}.
In Section~\ref{sec:conj}, we formally state the new conjecture (Conjecture~\ref{conj:main}) and discuss its implications.
We present theorems that prove Conjecture~\ref{conj:main} for special cases in Section~\ref{sec:proof-of-conj}; their proofs are sketched in Section~\ref{sec:proof-sketches}.
%We also provide further evidence for the conjecture from the positivstellensatz of matrix polynomials \citep{lai2020recht} in Section~\ref{sec:positivstellensatz}. 
\begin{figure}[tbp]
     \centering
     \begin{subfigure}[b]{0.47\textwidth}
         \centering
         \includegraphics[width=\textwidth]{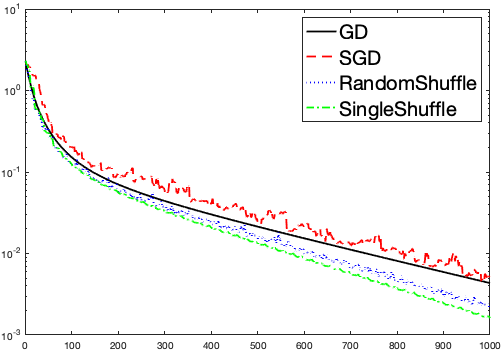}
         \caption{Convergence of $F(\vz_t)$}
         \label{fig:linearreg-loss}
     \end{subfigure}
     \hfill
     \begin{subfigure}[b]{0.47\textwidth}
         \centering
         \includegraphics[width=\textwidth]{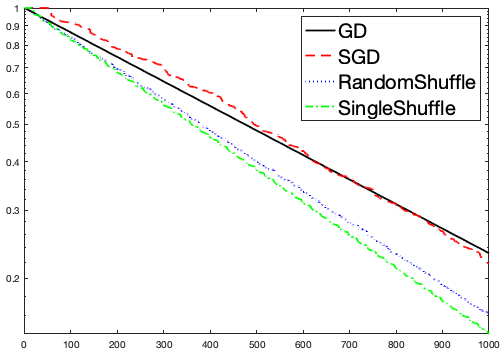}
         \caption{Evolution of spectral norm}
         \label{fig:linearreg-norm}
     \end{subfigure}
        \caption{Convergence of the training loss and the norm of corresponding matrices (see Section~\ref{sec:empirical}) in a linear regression problem with $n = 20$ data points in $d = 30$ dimension, for $K = 50$ epochs.}
        \label{fig:linearreg}
\end{figure}

% \begin{figure}[tbp]
% \floatconts
%   {fig:linearreg}
%   {\caption{Convergence of the training loss and the norm of corresponding matrices (see Section~\ref{sec:empirical}) in a linear regression problem with $n = 20$ data points in $d = 30$ dimension, for $K = 50$ epochs.}}
%   {%
%     \subfigure[Convergence of $F(\vz_t)$]{\label{fig:linearreg-loss}%
%       \includegraphics[width=0.47\linewidth]{images/coltplotv3-1.png}}%
%     \qquad
%     \subfigure[Evolution of spectral norm]{\label{fig:linearreg-norm}%
%       \includegraphics[width=0.47\linewidth]{images/coltplotv3-2.png}}
%   }
% \end{figure}

%%% Local Variables:
%%% mode: latex
%%% TeX-master: "main"
%%% End:

%% file: 2_background_arxiv.tex
\section{Background and motivation}
\label{sec:background}
%\subsection{Notation}
\paragraph{Notation.}
We first summarize notation used in this paper. For a positive integer $a$, we use $[a] \defeq \{1, 2, \dots, a\}$. Let $\mc S_n$ be the symmetric group of $[n]$, i.e., the set of all permutations $\sigma: [n] \to [n]$.
For a vector $\vv \in \reals^d$, we use $\norm{\vv}$ to denote its Euclidean ($\ell_2$) norm. Let $[\vv]_i$ be the $i$-th component of a vector $\vv$. We write a $d$-dimensional all-one vector as $\ones_d$.
Given a matrix $\mA \in \reals^{d \times k}$, we use $\norm{\mA}$ to denote its spectral norm, and $[\mA]_{i,j}$ to denote its $(i,j)$-th entry. We define $\sS^d$ to be the set of $d$-by-$d$ real symmetric matrices. We also use $\sS^d_{+}$ (or $\sS^d_{++}$) to denote the set of $d$-by-$d$ real symmetric PSD (or PD) matrices. For two symmetric matrices $\mA, \mB \in \sS^d$, $\mA \preceq \mB$ (or $\mA \prec \mB$) means that $\mB - \mA \in \sS^d_+$ (or $\mB - \mA \in \sS^d_{++})$.

\subsection{Connection of Conjectures \eqref{eq:recht-re-conj} and \eqref{eq:our-conj} to the convergence analysis of SGD}
\label{sec:moti-sgd}
In this section, we briefly introduce with- and without-replacement SGD algorithms to motivate both the original AM-GM conjecture~\eqref{eq:recht-re-conj} and our new conjecture~\eqref{eq:our-conj}. Suppose we want to solve a finite-sum optimization problem of the form
\begin{equation}
    \minimize\nolimits_{\vz \in \R^d} \quad F(\vz) \defeq \tfrac{1}{n} \sum\nolimits_{i=1}^n f_i(\vz).
\label{eq:finite-sum-prob}
\end{equation}
For simplicity of illustration, consider quadratic component functions $f_i(\vz) \defeq \tfrac{1}{2} \vz^T \mM_i \vz$, where $\mM_i \in \sS^d_+$.
We consider running SGD with a constant step-size $\eta > 0$, with the following updates:
\begin{equation*}
    \vz_t = \vz_{t-1} - \eta \nabla f_{i(t)} (\vz_{t-1}) 
    = \vz_{t-1} - \eta \mM_{i(t)} \vz_{t-1} = (\mI - \eta \mM_{i(t)}) \vz_{t-1},
\end{equation*}
where $i(t)$ is the index at iteration $t$ chosen by the algorithm. 

In \emph{with-replacement SGD}, $i(t)$ is drawn independently from the uniform distribution over $[n]$, making $\nabla f_{i(t)}$ an unbiased estimate of the full gradient $\nabla F$. Since $i(t)$'s are independent, we have
\begin{equation}
    \E[\vz_t] 
    = \E_{i(t)}[(\mI-\eta\mM_{i(t)})]\E[\vz_{t-1}]
    = \bigg (\mI- \frac{\eta}{n}\sum_{i=1}^n \mM_i \bigg )\E[\vz_{t-1}] = \bigg (\mI- \frac{\eta}{n}\sum_{i=1}^n \mM_i \bigg )^t \vz_0,
\label{eq:sgditer}
\end{equation}
where $\vz_0$ is the initialization. Note that this expectation $\E[\vz_t]$ is also equal to the GD iterate after $t$~steps. For the remaining of the paper, we will use $\sgd$ to denote the with-replacement SGD.

In contrast, in \emph{without-replacement SGD}, the algorithm shuffles the component functions $f_i$ by a random permutation $\sigma \in \mc S_n$ and makes a complete pass through the entire set of functions in the order $f_{\sigma(1)}, f_{\sigma(2)}, \dots, f_{\sigma(n)}$. This complete pass is called an \emph{epoch}. At the $t$-th iteration of an epoch, the index $\sigma(t)$ is used for the update, which can be written as 
\begin{equation*}
    \vz_{t} = \vz_{t-1} - \eta \nabla f_{\sigma(t)}(\vz_{t-1})
    = \big(\mI - \eta \mM_{\sigma(t)}\big) \vz_{t-1}.
\end{equation*}
If we unroll the recursion from $\vz_n$ to $\vz_0$ (i.e., over an epoch), we obtain
\begin{equation*}
    \vz_{n} = \Big( \prod\nolimits_{t=n}^1 \big(\mI - \eta \mM_{\sigma(t)}\big) \Big ) \vz_{0}.% \eqdef \mP_\sigma \vz_0.
\end{equation*}
To compute the expected value of $\vz_n$, we have to take an expectation over permutations $\sigma \sim {\rm Unif}(\mc S_n)$, so the expected value of the product $\prod\nolimits_{t=n}^1 (\mI - \eta \mM_{\sigma(t)})$ is of interest to us:
\begin{equation}
    \E_\sigma [\vz_n] = 
    \E_\sigma \left [ \prod\nolimits_{t=n}^1 (\mI - \eta \mM_{\sigma(t)}) \right ] \vz_0
    = \left ( \frac{1}{n!} \sum\nolimits_{\sigma \in \mc S_n} \prod\nolimits_{i=1}^n (\mI - \eta \mM_{\sigma(i)}) \right ) \vz_0.
\label{eq:rs1epoch}
\end{equation}
Comparing the matrices in the RHSs of \eqref{eq:sgditer} and \eqref{eq:rs1epoch}, we can notice their connection to the original AM-GM inequality conjecture~\eqref{eq:recht-re-conj}, with $\mA_i = \mI - \eta \mM_i$.
If true, inequality~\eqref{eq:recht-re-conj} would imply that in the special case $f_i(\vz) = \half \vz^T \mM_i \vz$, the expected iterate of without-replacement SGD after an epoch is closer to the global minimum than that of $\sgd$ or GD.

Without-replacement SGD is typically run over multiple epochs, and the algorithm has variants depending on how the permutation is selected for different epochs. An algorithm we refer to as $\randshuf$ takes fresh independent samples of $\sigma_k\sim {\rm Unif}(\mc S_n)$ at each epoch $k$. Another algorithm called $\singshuf$ samples a permutation $\sigma \sim {\rm Unif}(\mc S_n)$ in the beginning of the first epoch and then reuses this particular $\sigma$ for all the epochs. 
This difference in the shuffling scheme results in different expected iterates after running the algorithms for multiple epochs. 

To simplify notation, we define $\mP_{\sigma}\defeq \prod\nolimits_{i=1}^n (\mI - \eta \mM_{\sigma(i)})$. 
In case of $\randshuf$, the expected iterate after $nK$ iterates, i.e., $K$ epochs, reads
\begin{equation}
    \E[\vz^{\rm RS}_{nK}] = \E_{\sigma_K} [ \mP_{\sigma_K} ] \cdots \E_{\sigma_1} [ \mP_{\sigma_1} ] \vz_0
    = 
    \left ( \frac{1}{n!} \sum\nolimits_{\sigma \in \mc S_n} \mP_\sigma \right )^K \vz_0,
\label{eq:rsKepoch}
\end{equation}
because $\sigma_k$ are mutually independent. In $\singshuf$, the permutation $\sigma$ is fixed, so
\begin{equation}
    \E[\vz^{\rm SS}_{nK}] = \E_{\sigma} [ \mP_\sigma^K ] \vz_0
    = 
    \left ( \frac{1}{n!} \sum\nolimits_{\sigma \in \mc S_n} \mP_\sigma^K \right ) \vz_0.
\label{eq:ssKepoch}
\end{equation}
Notice that the matrix on the RHS of \eqref{eq:rsKepoch} is a power of an expectation, whereas the other one in \eqref{eq:ssKepoch} is an expectation of powers.
Comparing \eqref{eq:sgditer}, \eqref{eq:rsKepoch}, and \eqref{eq:ssKepoch}, the implication of our new conjecture~\eqref{eq:our-conj} on the expected iterates of without-replacement SGD algorithms becomes clear.

In many recent theoretical advances on without-replacement SGD~\citep{haochen2018random,nagaraj2019sgd,rajput2020closing,ahn2020sgd},
$\eta = O(\frac{\log(nK)}{nK})$ is typically chosen as the step-size.
This choice makes the matrices $\mA_i = \mI - \eta \mM_i$ close to identity; or equivalently, they have small condition numbers. Therefore, in the context of without-replacement SGD, it is evident that understanding the conjecture~\eqref{eq:our-conj} for well-conditioned matrices is of great importance.

\subsection{Motivation: S{\small INGLE}S{\small HUFFLE} beats R{\small ANDOM}S{\small HUFFLE} and S{\small GD} in linear regression}
\label{sec:empirical}
Now that we have defined with- and without-replacement SGD algorithms, we report an interesting empirical observation that serves as the main motivation for our new conjecture~\eqref{eq:our-conj}.

Consider solving a linear regression problem, which is a canonical example of finite-sum optimization problems~\eqref{eq:finite-sum-prob}. We are given $n$ data points $\{(\vx_i, y_i)\}_{i=1}^n$, where $\vx_i \in \reals^d$ is a unit vector and $y_i \in \reals$. We would like to minimize $F(\vz) = \half \sum_{i=1}^n (\vx_i^T\vz - y_i)^2$.

Using data points sampled i.i.d.\ from the standard Gaussian distribution, we tested convergence speed of GD, $\sgd$, $\randshuf$ and $\singshuf$, using the same random initialization and step-size $\eta = 0.5$ across all four algorithms. Rather surprisingly, plotting the objective function values $F(\vz_t)$ over iterations reveals that $\singshuf$ quite consistently outperforms $\randshuf$, and $\randshuf$ outperforms $\sgd$ and GD. Figure~\ref{fig:linearreg}(a) shows a typical plot obtained with $n = 20$ and $d = 30$, over $K = 50$ epochs. Due to the randomness in data points, initialization, and choice of permutations, there are some cases where $\singshuf$ performs worse than other methods, but in most of the runs we obtained plots similar to Figure~\ref{fig:linearreg}(a). We also find that this tendency is much stronger in underdetermined settings ($d > n$). 

We observe the same trend with $y_i = 0$ for all $i \in [n]$ as well. Note that linear regression with labels $y_i = 0$ corresponds to the case $f_i(\vz) = \half \vz^T \mM_i \vz$ discussed in Section~\ref{sec:moti-sgd}, with $\mM_i = \vx_i \vx_i^T$.
Recall that the $t$-th iterate of SGD on this problem is $\vz_t = \prod_{j=t}^1(\mI - \eta \mM_{i(j)}) \vz_0$, where the choice of index $i(j)$ depends on the sampling scheme. To see if the difference in convergence speed can be attributed to the difference in the decay speed of the spectral norm of the matrix $\prod_{j=t}^1(\mI - \eta \mM_{i(j)})$ over iterations, we plot the spectral norm $\norms{\mV \prod_{j=t}^1(\mI - \eta \mM_{i(j)})}$ over iterations in Figure~\ref{fig:linearreg}(b)\footnote{We plot $\norms{\mV(\mI - \frac{\eta}{n}\sum_{i=1}^n \mM_i)}^{nK}$ for GD.}. Here, $\mV$ is the projection matrix onto $\spann\{\vx_1, \dots, \vx_n\}$ because the spectral norm is always fixed to 1 when the regression problem is underdetermined ($d > n$). 
We observe that the decay speed of spectral norm agrees with the convergence speed of $F(\vz_t)$: in most runs, the spectral norm decays the fastest for $\singshuf$, followed by $\randshuf$, and then $\sgd$/GD. 

This inspiring observation motivates an extension of the AM-GM inequality conjecture~\eqref{eq:recht-re-conj}, which corresponds to ``$\randshuf \leq \sgd$/GD,'' to our new conjecture~\eqref{eq:our-conj} by adding a new inequality ``$\singshuf \leq \randshuf$.''

\subsection{Motivation: Disproofs of the Recht-R\'e conjecture break for well-conditioned matrices}
\label{sec:breaking-disproofs}
Before we formally state the new conjecture~\eqref{eq:our-conj}, we discuss the reason why it makes sense to pursue a proof of the AM-GM inequality (which is the second inequality of \eqref{eq:our-conj}) when its disproofs are already known \citep{lai2020recht,de2020random}. Essentially, the reason is that the known disproofs only work for PSD matrices or PD matrices with large condition numbers; hence, there is a possibility that the AM-GM inequality is still true for PD matrices with smaller condition numbers, which are more relevant to the convergence analysis of without-replacement methods.

\paragraph{Disproof by Positivstellensatz.}
\citet{lai2020recht} disprove the AM-GM inequality~\eqref{eq:recht-re-conj} by formulating the Positivstellensatz of the matrix polynomials from an equivalent conjecture into a semidefinite program (SDP). The equivalent conjecture reads: 
\begin{quote}
\textit{For any $\mA_1, \dots, \mA_n$ satisfying $\mA_i \succeq \zeros$ for all $i \in [n]$ and $\sum_{i=1}^n \mA_i \preceq n\mI$, we have
%\begin{equation*}
    $-n! \mI \preceq \sum_{\sigma \in \mc S_n} \prod_{i=1}^n \mA_{\sigma(i)} \preceq n! \mI$.}
%\end{equation*}}
\end{quote}
Using the SDP formulation, \citet{lai2020recht} calculate the minimum $\lambda_1$ and $\lambda_2$ that satisfy
\begin{equation}
\label{eq:lailim}
    \lambda_1 \mI - \sum\nolimits_{\sigma \in \mc S_n} \prod\nolimits_{i=1}^n \mA_{\sigma(i)} \succeq \zeros,~~\text{ and }~~
    \lambda_2 \mI + \sum\nolimits_{\sigma \in \mc S_n} \prod\nolimits_{i=1}^n \mA_{\sigma(i)} \succeq \zeros,
\end{equation}
for all $(\mA_1, \dots \mA_n) \in \mc A \defeq \{(\mA_1, \dots, \mA_n) : \mA_i \succeq \zeros, \text{ and } \sum_i \mA_i \preceq n\mI \}$ (see the paper for more details).
From the SDP formulation, \citet{lai2020recht} find out that $\lambda_2 \approx 144.65$ for $n = 5$, which exceeds $5! = 120$, hence disproving the conjecture for $n = 5$\footnote{The authors also checked carefully that it is not a numerical error.}.

To see if the inequality~\eqref{eq:recht-re-conj} holds for PD matrices with smaller condition numbers, we slightly changed the code provided by \citet{lai2020recht} to calculate the minimum $\lambda_1$ and $\lambda_2$~\eqref{eq:lailim}, this time over a more stringent set $\wt{\mc A}_\eta \defeq \{(\mA_1, \dots, \mA_n) : \mA_i \succeq (1-\eta)\mI, \text{ and } \sum_i \mA_i \preceq n\mI \}$. 
We note that $\wt{\mc A}_1$ corresponds to the original $\mc A$.
%Note that this set $\wt{\mc A}_\eta$ is a superset of the set of matrices meeting the requirement $(1-\eta) \mI \preceq \mA_i \preceq \mI$ in Conjecture~\ref{conj:main}.

% \begin{table*}[tbp]
%   \label{tab:sdp}%
%   \caption{Minimum Values of $\lambda_2$ \eqref{eq:lailim} for different values of $n$ and $\eta$}%
%   {\centering
% \begin{tabular}{r|r|r|r|r|r|r|r}
% \multicolumn{1}{c|}{$n$} &
% \multicolumn{1}{c|}{$n!$} &
% \multicolumn{1}{c|}{$\eta = 1$} &
% \multicolumn{1}{c|}{$\eta = 0.9$} &
% \multicolumn{1}{c|}{$\eta = 0.8$} &
% \multicolumn{1}{c|}{$\eta = 0.7$} &
% \multicolumn{1}{c|}{$\eta = 0.6$} &
% \multicolumn{1}{c}{$\eta = 0.5$} \\
% \hline
% 2 & 2 & 0.5000 & 0.2250 & 2.321e-10 & -0.1750 & -0.3200 & -0.5000 \\
% 3 & 6 & 3.4113 & 2.0147 & 0.9434 & 0.1620 & -0.3639 & -0.7500 \\
% 4 & 24 & 22.4746 & 15.1451 & 8.6081 & 3.6744 & 0.4138 & -1.4956 \\
% 5 & 120 & \textbf{144.6488} & 105.4259 & 70.0519 & 38.1546 & 13.2151 & -1.5622 \\
% \end{tabular}
% }
% \end{table*}

\begin{table}[tbp]
\centering
\begin{tabular}{r|r|r|r|r|r|r|r}
\multicolumn{1}{c|}{$n$} &
\multicolumn{1}{c|}{$n!$} &
\multicolumn{1}{c|}{$\eta = 1$} &
\multicolumn{1}{c|}{$\eta = 0.9$} &
\multicolumn{1}{c|}{$\eta = 0.8$} &
\multicolumn{1}{c|}{$\eta = 0.7$} &
\multicolumn{1}{c|}{$\eta = 0.6$} &
\multicolumn{1}{c}{$\eta = 0.5$} \\
\hline
2 & 2 & 0.5000 & 0.2250 & 2.321e-10 & -0.1750 & -0.3200 & -0.5000 \\
3 & 6 & 3.4113 & 2.0147 & 0.9434 & 0.1620 & -0.3639 & -0.7500 \\
4 & 24 & 22.4746 & 15.1451 & 8.6081 & 3.6744 & 0.4138 & -1.4956 \\
5 & 120 & \textbf{144.6488} & 105.4259 & 70.0519 & 38.1546 & 13.2151 & -1.5622 \\
\end{tabular}
\caption{Minimum Values of $\lambda_2$ \eqref{eq:lailim} for different values of $n$ and $\eta$}
\label{tab:sdp}
\end{table}

Table~\ref{tab:sdp} shows the values of minimum $\lambda_2$ obtained for $\wt{\mc A}_\eta$, for $\eta \in \{0.5, 0.6, \dots, 1\}$\footnote{We do not report the values of $\lambda_1$ because the calculated values of $\lambda_1$ were always equal to $n!$, regardless of $\eta$; this is in fact attained by setting $\mA_1 = \dots = \mA_n = \mI$.}.
A negative value of the minimum $\lambda_2$ means that the sum $\sum_{\sigma} \prod_i \mA_{\sigma(i)}$ is in fact PD for all matrices in $\wt{\mc A}_\eta$.
One can see in Table~\ref{tab:sdp} that in case of $n = 5$, although we have $\lambda_2 > 5!$ for $\eta = 1$, $\lambda_2$ quickly becomes smaller than $5!$ with $\eta = 0.9$. This suggests that by letting $\eta$ decrease with $n$, we can still make the AM-GM inequality~\eqref{eq:recht-re-conj} hold.
We actually prove later in Theorem~\ref{thm:rs-1} (Section~\ref{sec:proof-of-conj}) that \eqref{eq:recht-re-conj} holds in some linear regression settings for general $n$ and $\eta = O(\frac{1}{n})$, providing evidence to this observation.
%this observation is also consistent with our partial proof of our new conjecture (Theorem~\ref{thm:rs-1}) that we present later.
%While it seems from Table~\ref{tab:sdp} that we need to decrease $\eta$ with $n$ to make the inequality hold, note also that the minimum values of $\lambda_2$ obtained for $\eta = 0.5$ are negative and decreasing with $n$, suggesting that the AM-GM inequality may even hold with $\eta = 0.5$ for all $n$---a step-size constant independent of $n$.

As also noted in \citet{lai2020recht}, the limitation of the SDP formulation is that the size of the SDP explodes quickly with $n$, so already the SDP with $n = 6$ exceeds the memory size of a moderate-sized computer. We were only able to solve the SDPs up to $n = 5$.

\paragraph{Disproof by counterexamples.}
\citet{de2020random} disproves the AM-GM inequality by \emph{constructing} a class of counterexamples. For any $n \geq 2$, \citet{de2020random} constructs $n$ matrices $\mA_i = \mI + \ones_n \vy_i^T + \vy_i \ones_n^T \in \sS^n_+$ for $i \in [n]$, where $\vy_i$ is defined as $[\vy_i]_i = \frac{\sqrt{n-1}}{n}$ and $[\vy_i]_j = \frac{-1}{n\sqrt{n-1}}$ for any $j \neq i$.
%\begin{equation*}
    % [\vy_i]_j = \begin{cases}
    % \frac{\sqrt{n-1}}{n} & \text{for } j=i\\
    % \frac{-1}{n\sqrt{n-1}} & \text{for } j \neq i.
    % \end{cases}
%\end{equation*}
It is shown in \citet{de2020random} that this set of $n$ matrices is a counterexample to \eqref{eq:recht-re-conj} if $n$ satisfies
\begin{equation*}
    \left | \left( 1+\tfrac{1}{n-1} \right)^{n/2} \cos\left( n \arcsin\left( \tfrac{1}{\sqrt{n}} \right)\right )\right | > 1.
\end{equation*}
Here, $n = 5$ is the smallest such $n$.

The counterexample matrices $\mA_i$ are rank-deficient PSD matrices, with $\norm{\mA_i} = 2$ for all $i \in [n]$. Interestingly, the rank-deficiency or high enough condition number of $\mA_i$ seems important in the construction of counterexamples, because adding small $c_n \mI$ to each $\mA_i$ makes the matrices no longer a counterexample. For example, for $n = 5$, adding $0.06 \mI$ to each $\mA_i$ breaks the counterexample. For $n = 10$, adding $0.36 \mI$ suffices. Again, this observation leaves open the possibility that the inequality~\eqref{eq:recht-re-conj} is true for sufficiently well-conditioned PD matrices.

\paragraph{Positive results for close-to-identity matrices.}
Recent theoretical advances further inspire the study of a revised AM-GM inequality conjecture~\eqref{eq:recht-re-conj} for PD matrices with small condition numbers.
Theorem~1 of \citet{de2020random} proves that for any $n \geq 2$, if matrices $\mM_1, \dots, \mM_n \in \sS^d$ have no eigenvalue-eigenvector pair shared across all $\mM_i$'s, then there exists $\eta_{\max} \in (0,1]$ such that for any $\eta \in [0,\eta_{\max}]$ and $\mA_i = \mI-\eta \mM_i$, the inequality \eqref{eq:recht-re-conj} is true.
\citet{ahn2020sgd} also show that a slightly looser version of the \emph{symmetrized} AM-GM inequality conjecture (see \eqref{eq:otherver-4} in Appendix~\ref{sec:othervers} for its definition) holds for $\eta = O(\frac{1}{n})$ and matrices $\mA_i = \mI - \eta \mM_i$, where $\norm{\mM_i} \leq 1$ for all $i \in [n]$.

%% file: 3_conjecture_arxiv.tex
\section{New conjecture: SS-RS-GD inequalities}
\label{sec:conj}
Based on the motivations discussed so far, we now formally state our new conjecture in Conjecture~\ref{conj:main}, which we refer to as the \emph{SS-RS-GD inequalities} conjecture. Conjecture~\ref{conj:main} extends the AM-GM inequality conjecture~\eqref{eq:recht-re-conj} by adding the $\singshuf$ algorithm to the picture. It also refines the original conjecture to a setting that is more relevant to the convergence analysis of SGD.
\begin{conjecture}
\label{conj:main}
For any $n \geq 2$, $K \geq 1$, and $d \geq 1$, there exists a step-size constant $\eta_{n,K} \in (0,1]$ such that the following statement holds:
Suppose $\mA_1, \dots, \mA_n \in \sS^d_{+}$ 
%are symmetric positive definite matrices satisfying
satisfy
$(1-\eta_{n,K}) \mI \preceq \mA_i \preceq \mI$ for all $i \in [n]$. Then, for the three matrices defined as
\begin{align}
\label{eq:def-means}
    \ssmean \defeq \frac{1}{n!} \!\sum_{\sigma \in \mc S_n} \!\left( \prod_{i=1}^n \mA_{\sigma(i)} \right)^K\!,
    \rsmean \defeq \left( \frac{1}{n!} \!\sum_{\sigma \in \mc S_n} \prod_{i=1}^n \mA_{\sigma(i)} \right)^K\!,
    \gdmean \defeq \left( \frac{1}{n} \sum_{i=1}^n \mA_i \right)^{nK}\!,
\end{align}
the following spectral norm inequalities are satisfied:
\begin{equation*}
    \norm{\ssmean} \leq \norm{\rsmean} \leq \norm{\gdmean}.
\end{equation*}

% Then, we have
% \begin{align*}
%     \norm{\frac{1}{n!} \sum_{\sigma \in \mc S_n} \left( \prod_{i=1}^n \mA_{\sigma(i)} \right)^K} \leq 
%     \norm{\left( \frac{1}{n!} \sum_{\sigma \in \mc S_n} \prod_{i=1}^n \mA_{\sigma(i)} \right)^K} \leq
%     \norm{\left( \frac{1}{n} \sum_{i=1}^n \mA_i \right)^{nK}}.
% \end{align*}
\end{conjecture}
Several remarks on Conjecture~\ref{conj:main} are in order.

\vspace*{-2pt}
\paragraph{Requirements on ``scaling'' are not strict.} 
Although the requirement $(1-\eta_{n,K}) \mI \preceq \mA_i \preceq \mI$ in Conjecture~\ref{conj:main} assumes $\mA_i$ in the form of the identity matrix minus a small PSD matrix, this specific form is not strictly necessary because the inequalities do not change when we scale all the matrices with the same multiplicative factor. Also, for the first inequality of the conjecture, individual scaling of $\mA_i$ is not important because scaling a single matrix $\mA_i$ does not change the inequality $\norm{\ssmean} \leq \norm{\rsmean}$. Rather, the actual requirement on $\mA_i$'s for the first inequality is a restriction that they are PD with condition numbers at most $\frac{1}{1-\eta_{n,K}}$.

\vspace*{-2pt}
\paragraph{Matrices that commute.}
For matrices $\mA_1, \dots, \mA_n$ that commute, Conjecture~\ref{conj:main} is true with $\eta_{n,K} = 1$. The product $\prod_{i=1}^n \mA_{\sigma(i)}$ is identical for all $\sigma \in \mc S_n$, so the first inequality of the conjecture holds with equality. The second inequality boils down to the scalar AM-GM inequality.

\vspace*{-2pt}
\paragraph{Doesn't Conjecture~\ref{conj:main} contradict matrix convexity of $\mA \mapsto \mA^2$?}
While the $\norm{\rsmean} \leq \norm{\gdmean}$ inequality of Conjecture~\ref{conj:main} is just the $K$-th power of the AM-GM inequality conjecture~\eqref{eq:recht-re-conj} with more stringent requirements on $\mA_i$'s, the first inequality $\norm{\ssmean} \leq \norm{\rsmean}$ may look questionable at first glance.
It is of the form $\norms{\E_\sigma [\mP_\sigma^K]} \leq \norms{\E_\sigma [\mP_\sigma]^K}$ and the direction of the inequality is the \textbf{\emph{opposite}} of Jensen's inequality! 
In particular, for $K=2$, it is well-known that the map $\mA \mapsto \mA^2$ on $\sS^d$ is matrix convex, meaning that we have $\frac{1}{4}(\mA+\mB)^2 \preceq \half(\mA^2 + \mB^2)$ for any $\mA, \mB \in \sS^d$. If $\mA$ and $\mB$ are also PSD, matrix convexity implies $\norm{\frac{1}{4}(\mA+\mB)^2} \leq \norm{\half(\mA^2 + \mB^2)}$, a norm inequality in the opposite direction from Conjecture~\ref{conj:main}.
The key difference here is that although $\mA_i$'s are symmetric and PD, their product $\prod_{i=1}^n \mA_{\sigma(i)}$ is not necessarily symmetric in general. Thus, Conjecture~\ref{conj:main} does not contradict the matrix convexity of $\mA \mapsto \mA^2$ on $\sS^d$.

\vspace*{-2pt}
\paragraph{Does Conjecture~\ref{conj:main} alone imply that $\singshuf$ always converges faster?}
The short answer is, in general, no. As seen in Section~\ref{sec:moti-sgd}, for the special case of $f_i(\vz) = \half \vz^T \mM_i \vz$, the conjectured inequalities indeed imply that the expected iterate of $\singshuf$ after $nK$ iterations is closer to the global minimum than $\randshuf$, $\sgd$, and GD. However, for general quadratic functions $f_i(\vz) = \half \vz^T \mM_i \vz + \vb_i^T \vz + c_i$, the linear coefficients $\vb_i$ introduce additional ``noise'' terms and they result in SGD iterates of the form
%\vspace*{-3pt}
\begin{equation*}
    \vz_{t} = \Big ( \prod\nolimits_{j=t}^1 (\mI - \eta \mM_{i(j)}) \Big ) \vz_0
    - \eta \vb_{i(t)} - \eta \Big ( \sum\nolimits_{l=1}^{t-1} \Big ( \prod\nolimits_{j=t}^{l+1} (\mI - \eta \mM_{i(j)} )\Big ) \vb_{i(l)} \Big ),
    %\vspace*{-2pt}
\end{equation*}
and the terms involving $\vb_i$'s become a dominant factor that ``slows down'' SGD and determines the convergence rate. For example, if we assume without loss of generality that $\sum_{i=1}^n \vb_i = \zeros$ so as to locate the global minimum at $\vz^* = \zeros$, the GD iterates are $\vz_t = (\mI - \frac{\eta}{n} \sum_{i=1}^n \mM_i)^t \vz_0$. Notice that the iterates do not have any terms involving $\vb_i$'s, which is contrary to SGD; this difference results in a linear convergence rate of GD and sublinear rates of SGD algorithms.
Although Conjecture~\ref{conj:main} alone does not prove superior performance of $\singshuf$ over other algorithms, it still provides useful tools in the analysis of without-replacement sampling methods. We also believe that the matrix norm inequalities themselves are interesting in their own right!

\vspace*{-2pt}
\paragraph{Independence of $\eta_{n,K}$ from $d$ and $\mA_i$'s.}
In the step-size constant $\eta_{n,K}$ in Conjecture~\ref{conj:main}, we did not use the dimension $d$ in the subscript to reflect our belief that a dimension-independent constant exists. Also, Conjecture~\ref{conj:main} is stated in a way that $\eta_{n,K}$ is also matrix-independent, i.e., $\eta_{n,K}$ does not depend on $\mA_i$'s. Theorems~\ref{thm:ss-2} and \ref{thm:rs-1} presented in Section~\ref{sec:proof-of-conj} support our belief: we prove special cases of the conjecture with step-size constants $\eta_{n,K} = O(\frac{1}{K})$ or $O(\frac{1}{n})$, which are \emph{independent} of $d$ and $\mA_i$'s.
In fact, these theorems suggest that the true $\eta_{n,K}$ satisfying Conjecture~\ref{conj:main} may be of order $O(\frac{1}{nK})$. 
Although this may look like a strong restriction, we emphasize that this matches (up to log factors) the order of the step-sizes $\eta$ chosen in many recent results on without-replacement SGD~\citep{haochen2018random,nagaraj2019sgd,rajput2020closing,ahn2020sgd}.

\vspace*{-2pt}
\paragraph{Should $\eta_{n,K}$ necessarily decay with $n$ and $K$?}
From Section~\ref{sec:breaking-disproofs}, we saw that $\eta_{n,K} < 1$ must hold for $n \geq 5$ to make $\norm{\rsmean} \leq \norm{\gdmean}$ hold, and we speculated that $\eta_{n,K}$ should decrease with $n$ for the inequality to hold.
For the other inequality $\norm{\ssmean} \leq \norm{\rsmean}$, we do have a counterexample for any $n \geq 3$, $K \geq 2$ and $d \geq 2$ if $\eta_{n,K} = 1$; we prove this in Appendix~\ref{sec:finding-counterex}. The presence of a counterexample indicates that $\eta_{n,K} < 1$ is necessary for $n \geq 3$, again motivating our assumption of well-conditioned PD matrices. 
A further investigation on this counterexample also provides an evidence that $\eta_{n,K}$ should shrink with $K$ for $\norm{\ssmean} \leq \norm{\rsmean}$ to hold; see Appendix~\ref{sec:finding-counterex} for details.

% used a simple random procedure to find counterexamples to the inequality, and it indeed found counterexamples for $n \geq 3$ and $K \geq 2$ when $\eta_{n,K} = 1$; however, the same procedure did not succeed in finding counterexamples when we set $\eta_{n,K} = 0.5$. This indicates that the step-size constant should necessarily satisfy $\eta_{n,K} < 1$ for large $n$ and $K$, further motivating the study of matrices with small condition numbers.
% See Appendix~\ref{sec:finding-counterex} for more details on our procedure and counterexamples.

\vspace*{-2pt}
\paragraph{Can one use the Positivstellensatz to verify the first inequality?}
Unfortunately, we do not know if we can check the inequality $\norm{\ssmean} \leq \norm{\rsmean}$ with a Positivstellensatz and SDP approach similar to \citet{lai2020recht}. This is because in the equivalent conjecture for this inequality, the feasible set (corresponding to $\mc A$ and $\wt{\mc A}_\eta$ in Section~\ref{sec:breaking-disproofs}) is no longer representable with linear matrix inequalities, violating the requirements of the Positivstellensatz in~\citet{helton2012convex}. For example, for $n = 2$ and $\eta_{2,K} = 1$, the equivalent conjecture to $\norm{\ssmean} \leq \norm{\rsmean}$ reads:
%\vspace*{-2.5pt}
\begin{quote}
\textit{For any $\mA$, $\mB$ satisfying $\mA \succeq \zeros$, $\mB \succeq \zeros$ and $-2 \mI \preceq \mA \mB + \mB \mA \preceq 2\mI$, we have $-2 \mI \preceq (\mA \mB)^K + (\mB \mA)^K \preceq 2\mI$,}
\end{quote}
%\vspace*{-2.5pt}
which involves nonlinear inequalities such as $\mA \mB + \mB \mA \preceq 2\mI$.

\vspace*{-2pt}
\paragraph{Other versions of Conjecture~\ref{conj:main}.}
The original Recht-R\'e conjecture~\eqref{eq:recht-re-conj} has other versions such as \emph{symmetrized} or \emph{expectation-of-norm}, and some of them are also disproved in \citet{de2020random}. We discuss the variants of \eqref{eq:recht-re-conj} as well as the corresponding versions of Conjecture~\ref{conj:main} in Appendix~\ref{sec:othervers}.
%{\todo my worry: disproving other versions of the conjecture may give an impression that Conj~\ref{conj:main} may also be easily disproved?}

%%% Local Variables:
%%% mode: latex
%%% TeX-master: "main"
%%% End:

%% file: 4_theorems_arxiv.tex
\section{Proof of Conjecture~\ref{conj:main} for special cases}
\label{sec:proof-of-conj}
In this section, we present theorems that prove some special cases of Conjecture~\ref{conj:main}. From this point onward, we occasionally write the matrices $\ssmean$, $\rsmean$, and $\gdmean$~\eqref{eq:def-means} as functions of $n$, $K$, and $\mA_1, \dots, \mA_n$ (e.g., $\ssmean(n,K,\mA_{1:n})$), whenever we need to make the dependency more explicit.

\subsection{The SS-RS-GD inequalities are true for generic matrices and sufficiently small $\eta$}
We first show that for any fixed generic matrices $\mM_1, \dots, \mM_n$, the conjectured inequalities are true for sufficiently small step-size $\eta$ and $\mA_i = \mI - \eta \mM_i$.
In Section~\ref{sec:proof-thm-ss-1}, we prove the following theorem:
\begin{theorem}
\label{thm:ss-1}
For any $n \geq 2$, suppose symmetric matrices $\mM_1, \dots, \mM_n \in \sS^{d}$ satisfy 
\begin{equation}
\label{eq:thm-ss-1}
    \sum\nolimits_{i, j \in [n], i \neq j}(\mM_i \mM_j - \mM_j \mM_i)^2 \prec \zeros.
\end{equation}
Then, for any $K \geq 1$, there exists $\eta_{\max} \in (0,1]$ such that for any $\eta \in [0,\eta_{\max}]$, the matrices $\mA_1, \dots, \mA_n$ defined as $\mA_i = \mI - \eta \mM_i$ satisfy $\norm{\ssmean(n,K,\mA_{1:n})} \leq \norm{\rsmean(n,K,\mA_{1:n})}$.
%$0 \preceq \ssmean \preceq \rsmean$.
\end{theorem}
The assumption~\eqref{eq:thm-ss-1} requires that the sum of squared commutators of each pair of matrices $\mM_i$ is negative definite. We note that $\mM_i\mM_j - \mM_j \mM_i$ is a skew-symmetric matrix, so each $(\mM_i\mM_j - \mM_j \mM_i)^2$ is already negative semidefinite; therefore, the only requirement of \eqref{eq:thm-ss-1} is that the sum is full-rank. This assumption holds for generic matrices, i.e., the set of matrices $\mM_1, \dots, \mM_n$ that do not satisfy \eqref{eq:thm-ss-1} has Lebesgue measure zero.

Since Theorem~1 of \citet{de2020random} also holds for generic matrices and sufficiently small $\eta$, Theorem~1 of \citet{de2020random} and our Theorem~\ref{thm:ss-1} together prove that for generic symmetric matrices $\mM_1, \dots, \mM_n \in \sS^d$, $K \geq 1$ and sufficiently small $\eta \geq 0$, the matrices $\mA_i = \mI - \eta \mM_i$ satisfy the conjectured SS-RS-GD inequalities
\begin{equation*}
    \norm{\ssmean(n,K,\mA_{1:n})} \leq \norm{\rsmean(n,K,\mA_{1:n})} \leq \norm{\gdmean(n,K,\mA_{1:n})}.
\end{equation*} 
We emphasize that this holds for any $n$, $K$, and $d$.
However, we remark that the statement proven here is somewhat \emph{different} from our original Conjecture~\ref{conj:main}, because the step-size constant $\eta_{\max}$ \emph{depends} on the matrices $\mM_i$.
As discussed in Section~\ref{sec:conj}, our aim is to prove the existence of matrix-independent step-size constants $\eta_{n,K}$; we will see such theorems in the following sections.

\subsection{Proof for special cases of $n = 2$ with matrix-independent $\eta_{n,K}$}
In this section, we prove the inequality $\norm{\ssmean} \leq \norm{\rsmean}$ for $n=2$, in some special cases. 
\begin{theorem}
\label{thm:ss-2}
Consider $K = 2^m$, and $d \geq 1$, where $m$ is a positive integer. For $\mA, \mB \in \sS^d_{++}$ satisfying $(1-\frac{1}{2K}) \mI \preceq \mA \preceq \mI$ and $(1-\frac{1}{2K}) \mI \preceq \mB \preceq \mI$, we have 
\begin{equation*}
\half \norm{(\mA \mB)^K + (\mB \mA)^K} 
\leq
\norm{\half(\mA \mB + \mB \mA)}^K.
\end{equation*}
\end{theorem}
\begin{theorem}
\label{thm:ss-3}
For any $\mA, \mB \in \sS^2_+$ and $K \geq 1$, we have 
% \begin{equation*}
% \norm{\tfrac{(\mA \mB)^K + (\mB \mA)^K}{2}} 
% \leq
% \norm{\tfrac{\mA \mB + \mB \mA}{2}}^K.
% \end{equation*}
\begin{equation*}
\half \norm{(\mA \mB)^K + (\mB \mA)^K} 
\leq
\norm{\half(\mA \mB + \mB \mA)}^K.
\end{equation*}
\end{theorem}
Theorem~\ref{thm:ss-2} proves the $\norm{\ssmean} \leq \norm{\rsmean}$ inequality for $n = 2$, $K = 2^m$, and $d \geq 1$, with step-size constants $\eta_{2,K} = \frac{1}{2K}$. Theorem~\ref{thm:ss-3} proves the same inequality for $n = 2$, $K \geq 1$, and $d = 2$, with $\eta_{2,K} = 1$ (i.e., the inequality holds for \emph{any PSD} matrices).
Since $\norm{\rsmean} \leq \norm{\gdmean}$ holds for $n = 2$ and any PSD matrices~\citep{recht2012toward}, the theorems prove the entire Conjecture~\ref{conj:main} for the special cases considered in them. The proofs of Theorem~\ref{thm:ss-2} and \ref{thm:ss-3} are sketched in Sections~\ref{sec:proof-thm-ss-2} and \ref{sec:proof-thm-ss-3}, respectively.

The restriction of Theorem~\ref{thm:ss-3} to $d = 2$ is an artifact of our proof technique that does not extend to $d \geq 3$ (see Section~\ref{sec:proof-thm-ss-3} for more details). Nevertheless, for $n = 2$, we believe that the conjecture itself is likely true for any $d$ and $K$, with step-size constants $\eta_{2,K} = 1$.

\subsection{Proof of the AM-GM inequality for general $n$ in linear regression settings}
Next, we prove the second inequality of Conjecture~\ref{conj:main}, i.e., the matrix AM-GM inequality, for some linear regression settings with any $n$ and $\eta_{n,K} = O(\frac{1}{n})$, when the data points are ``diverse enough.''
\begin{theorem}
\label{thm:rs-1}
For any $n \geq 2$, $K\geq 1$, and $d \leq n$, suppose that we have $n$ unit vectors $\vx_1, \dots, \vx_n \in \reals^d$ satisfying the following two assumptions:
\begin{enumerate}[label=(A\arabic*)]
    \item \label{assm:1} For any $i, j \in [n]$, $i \neq j$, we have $|\vx_i^T \vx_j| \leq n^{-1/2}$.
    \item \label{assm:2} For any unit vector $\vu \in \reals^d$, we have $n^{-1/4} \leq \sum_{i=1}^n (\vu^T\vx_i)^2 \leq n^{1/4}$.
\end{enumerate}
Then, for any $\eta \in [0,\frac{1}{6n}]$, the matrices $\mA_1, \dots, \mA_n$ defined as $\mA_i = \mI - \eta \vx_i\vx_i^T$ satisfy the second inequality of Conjecture~\ref{conj:main}: $\norm{\rsmean(n,K,\mA_{1:n})} \leq \norm{\gdmean(n,K,\mA_{1:n})}$.
%$0 \preceq \rsmean \preceq \gdmean$.
\end{theorem}
The proof of Theorem~\ref{thm:rs-1} is sketched in Section~\ref{sec:proof-thm-rs-1}. Notice that we are restricting our attention to $d \leq n$, because if $d > n$, the linear span of unit vectors $\vx_1, \dots, \vx_n \in \reals^d$ has dimension at most $n$, so the inequality trivially holds with $\norm{\rsmean} = \norm{\gdmean} = 1$.

Theorem~\ref{thm:rs-1} requires that the data points are sufficiently ``spread out'' in the $\reals^d$ space. As a sanity check, when $d = n$, any set of vectors that forms an orthonormal basis satisfies the assumptions.
The assumptions are motivated from the \emph{mutual incoherence} and \emph{restricted isometry property} assumptions commonly adopted in the compressed sensing literature. They are used in order to ensure that the ``information'' is distributed evenly over a measurement matrix, so that exact recovery of the true sparse signal is possible (for a recent summary, see \citet{rani2018systematic} and references therein).
The first assumption~\ref{assm:1} is precisely the mutual incoherence assumption and it requires any two columns of the data matrix $\mX = \begin{bmatrix} \vx_1 & \cdots & \vx_n \end{bmatrix}$ to have bounded inner product, and the other assumption~\ref{assm:2} requires that the matrix $\mX\mX^T$ is close to an isometry, which is motivated from the restricted isometry property. 
%Both assumptions are commonly adopted in compressed sensing to ensure that the ``information'' is distributed evenly over the measurement matrix so that exact recovery of the true signal is possible.
%{\todo add compressed sensing citations}

Under these assumptions, Theorem~\ref{thm:rs-1} proves that the matrix AM-GM inequality holds for linear regression settings with data points that are diverse enough. 
Although Theorem~\ref{thm:rs-1} is restricted to a special case where $\mA_i$ is the identity matrix minus a rank-1 matrix, it is interesting because it works for all $n \geq 2$, especially provided that the original conjecture~\eqref{eq:recht-re-conj} for PSD matrices is false as soon as $n = 5$. The theorem provides evidence towards the correctness of our revised conjecture for well-conditioned PD matrices.

%% file: 5_pfsketch_arxiv.tex
\section{Proof sketches}
\label{sec:proof-sketches}
\subsection{Proof of Theorem~\ref{thm:ss-1}}
\label{sec:proof-thm-ss-1}
Theorem~\ref{thm:ss-1} is a consequence of the following lemma, which we prove in Appendix~\ref{sec:proof-lem-ssprobdep}:
\begin{lemma}
\label{lem:ssprobdep}
For any $n\geq 2$ and $K \geq 1$, suppose we have $\mA_1, \dots, \mA_n \in \reals^{d \times d}$ of the form $\mA_i = \mI - \eta \mM_i$.
Then, the difference $\rsmean(n,K,\mA_{1:n}) - \ssmean(n,K,\mA_{1:n})$ is written as
\begin{align*}
\rsmean - \ssmean 
\succeq -\frac{\eta^4 K (K-1)}{4n \cdot n!} \sum_{i,j \in [n], i\neq j} (\mM_i \mM_j - \mM_j \mM_i)^2 + O(\eta^5).
\end{align*}
\end{lemma}
The proof of Lemma~\ref{lem:ssprobdep} arranges $\rsmean$ and $\ssmean$ in increasing order of degree of $\eta$ and finds out that they differ only in degrees at least 4. Moreover, the lemma shows that the sum of degree-4 terms in $\rsmean-\ssmean$ is in fact PD, given our assumption that the sum of squared commutators of $\mM_i$ and $\mM_j$ are full-rank.

Consequently, we can write 
%\begin{equation*}
    $\rsmean - \ssmean = \eta^4 \mC + O(\eta^5)$,
%\end{equation*}
where $\mC$ is a PD matrix. This implies that there exists $\eta_1 > 0$ such that for any $\eta \in [0, \eta_1]$, the difference stays PSD: $\rsmean - \ssmean \succeq \zeros$.

Lastly, note that when $\eta = 0$, we have $\mA_i = \mI$ for all $i\in[n]$ and $\ssmean(n,K,\mI, \dots, \mI) = \mI$ is of course PD. Therefore, there exists $\eta_2 > 0$ such that for any $\eta \in [0, \eta_2]$ and $\mA_i = \mI-\eta\mM_i$, $\ssmean(n,K,\mA_{1:n})$ is still PSD. 
Choosing $\eta_{\max} = \min \{\eta_1, \eta_2\}$ ensures that the inequality $\norm{\ssmean} \leq \norm{\rsmean}$ holds for any $\eta \in [0, \eta_{\max}]$.

\subsection{Proof of Theorem~\ref{thm:ss-2}}
\label{sec:proof-thm-ss-2}
For any $K = 2^m$ ($m \geq 1$) and $\mA, \mB$ satisfying the assumptions, we prove in the following lemma (whose proof is deferred to Appendix~\ref{sec:proof-lem-sspsd}) that $(\mA \mB)^K + (\mB \mA)^K \succeq \zeros$ holds:
\begin{lemma}
\label{lem:sspsd}
For any $K \geq 1$, consider $\mA, \mB \in \sS^d_{++}$ satisfying $(1-\frac{1}{2K}) \mI \preceq \mA \preceq \mI$ and $(1-\frac{1}{2K}) \mI \preceq \mB \preceq \mI$. Then, we have $(\mA \mB)^K + (\mB \mA)^K \succeq \zeros$.
\end{lemma}
Next, we have $\half((\mA\mB)^K + (\mB\mA)^K) \preceq \frac{1}{4}((\mA\mB)^{K/2} + (\mB\mA)^{K/2})^2$, by the following argument:
\begin{align*}
    \half((\mA\mB)^K + (\mB\mA)^K) - \tfrac{1}{4}((\mA\mB)^{K/2} + (\mB\mA)^{K/2})^2
    = \tfrac{1}{4}((\mA\mB)^{K/2} - (\mB\mA)^{K/2})^2 \preceq \zeros.
\end{align*}
The last inequality comes from the fact that the square of a skew-symmetric matrix is negative semidefinite. Putting the two matrix inequalities together implies that we have
%, for $(1-\frac{1}{2K}) \mI \preceq \mA \preceq \mI$ and $(1-\frac{1}{2K}) \mI \preceq \mB \preceq \mI$, we have
\begin{equation}
\label{eq:proof-sketch-1}
    \norm{\tfrac{(\mA\mB)^K + (\mB\mA)^K}{2}} \leq 
    \norm{\tfrac{(\mA\mB)^{K/2} + (\mB\mA)^{K/2}}{2}}^2.
\end{equation}
For $K = 2$ ($m=1$), \eqref{eq:proof-sketch-1} already proves the desired inequality. For larger $K = 2^m$ ($m > 1$), we can apply \eqref{eq:proof-sketch-1} recursively and finish the proof.

\subsection{Proof of Theorem~\ref{thm:ss-3}}
\label{sec:proof-thm-ss-3}
Since $\mA, \mB \in \sS^d_+$, the product $\mA \mB$ can be written as $\mA \mB = \mQ \mS \mQ^{-1}$ where $\mQ \in \reals^{2\times 2}$ is an invertible matrix and $\mS \in \reals^{2 \times 2}$ is a diagonal matrix with nonnegative entries (Theorem~7.6.1 of \citet{horn2012matrix}). Using this representation, the inequality can be written as
\begin{equation}
\label{eq:lem-numrad-counterex-2}
    \norm{\tfrac{\mQ \mS^K \mQ^{-1} + \mQ^{-T} \mS^K \mQ^T}{2}} 
    \leq 
    \norm{\tfrac{\mQ \mS \mQ^{-1} + \mQ^{-T} \mS \mQ^T}{2}}^K.
\end{equation}
Then, the proof is implied from the following lemma on the numerical radius\footnote{We refer the readers to \citet{horn1991topics, dragomir2013inequalities} for an overview of this quantity.} of a matrix
\begin{equation*}
w(\mM) \defeq \sup\,\{|\vv^H \mM \vv|: \vv \in \C^d, \norm{\vv} \leq 1 \},
\end{equation*}
and its property that $w(\mM^K) \leq w(\mM)^K$ (see \citet{pearcy1966elementary} for a proof).
\begin{lemma}
\label{lem:numrad}
For any invertible matrix $\mQ \in \reals^{2 \times 2}$ and any diagonal matrix $\mLambda \in \reals^{2 \times 2}$ with nonnegative entries, we have $w(\mQ \mLambda \mQ^{-1}) = \half \norm{\mQ \mLambda \mQ^{-1} + \mQ^{-T} \mLambda \mQ^{T}}$.
\end{lemma}
The proof of this technical lemma is presented in Appendix~\ref{sec:proof-lem-numrad}.
From Lemma~\ref{lem:numrad}, we have
\begin{equation*}
    w(\mQ \mS^K \mQ^{-1}) = \norm{\tfrac{\mQ \mS^K \mQ^{-1} + \mQ^{-T} \mS^K \mQ^T}{2}} 
    \leq 
    \norm{\tfrac{\mQ \mS \mQ^{-1} + \mQ^{-T} \mS \mQ^T}{2}}^K 
    = w(\mQ \mS \mQ^{-1})^K.
    \vspace{5pt}
\end{equation*}

\paragraph{Remark: Extension to $d > 2$.}
Unfortuately, the proof technique using Lemma~\ref{lem:numrad} cannot be directly extended to $d>2$ because the lemma is no longer true when $d = 3$.
A counterexample is
\begin{equation}
\label{eq:lem-numrad-counterex-1}
    \mQ \mS \mQ^{-1} =
    {\footnotesize
    \begin{bmatrix}
    3 & 0.05 & 0.3\\
    -0.05 & 3 & -0.1\\
    -0.3 & 0.1 & 2
    \end{bmatrix}
    }.
\end{equation}
The matrix on the RHS is diagonalizable with real positive eigenvalues and real eigenvectors, so there exist $\mQ$ and $\mS$ satisfying \eqref{eq:lem-numrad-counterex-1}. However, one can check that $w(\mQ \mS \mQ^{-1}) \approx 3.0004 > \half \norm{\mQ \mS \mQ^{-1} + \mQ^{-T} \mS \mQ^{T}} = 3$. Nevertheless, the norm inequality~\eqref{eq:lem-numrad-counterex-2} still holds for this $\mQ \mS \mQ^{-1}$.

\subsection{Proof of Theorem~\ref{thm:rs-1}}
\label{sec:proof-thm-rs-1}
Recall that $\norm{\rsmean}$ and $\norm{\gdmean}$ are the $K$-th powers of both sides of the AM-GM inequality conjecture~\eqref{eq:recht-re-conj}. Hence, it suffices to prove the theorem for $K = 1$. Moreover, since the AM-GM inequality~\eqref{eq:recht-re-conj} is already true for $n=2$ and $3$, we only consider $n \geq 4$.

We prove Theorem~\ref{thm:rs-1} by the following technical lemma, proven in Appendix~\ref{sec:proof-lem-amgm}:
\begin{lemma}
\label{lem:amgm}
Suppose that we have $n$ unit vectors $\vx_1, \dots, \vx_n \in \reals^{d}$ ($d \leq n$) that satisfy $|\vx_i^T \vx_j| \leq \delta$ for all $i, j \in [n]$ such that $i \neq j$. Assume also that the maximum and minimum eigenvalues of $\sum_{i=1}^n \vx_i \vx_i^T$ are $s_{\max}$ and $s_{\min}$, respectively. For matrices $\mA_i \defeq \mI - \eta \vx_i\vx_i^T$,
% \begin{equation*}
%     \gdmean \defeq \left (\mI - \frac{\eta}{n} \sum_{i=1}^n \vx_i \vx_i^T \right )^n ~~\text{ and }~~
%     \rsmean \defeq \frac{1}{n!} \sum_{\sigma \in \mc S_n} \prod_{i=1}^n \left (\mI - \eta \vx_{\sigma(i)}\vx_{\sigma(i)}^T \right ),
% \end{equation*}
the following inequalities on minimum eigenvalues $\lambda_{\min} (\cdot)$ hold:
\begin{align}
    \lambda_{\min}(\gdmean - \rsmean) &\geq \frac{\eta^2(n-1)(1-\delta)s_{\min}}{2n}
    -
    \sum\nolimits_{m=4}^n \eta^m (s_{\max}^m + s_{\max} n^{m-1} \delta^{m-1} )\notag
    \\
    &~~~~ -\frac{\eta^3}{6}
    \left (
    \frac{3s_{\max}^3}{n}
    +
    s_{\max} n^{1/2}(n^2\delta^4+6n\delta^2+1)^{1/2}
    \right )\label{eq:proof-sketch-2},\\
    \lambda_{\min}(\rsmean) &\geq 1 - s_{\max} \sum\nolimits_{m=1}^n \eta^m n^{m-1} \delta^{m-1}.
    \label{eq:proof-sketch-3}
\end{align}
% \begin{align}
%     \label{eq:proof-sketch-2}
%     \lambda_{\min}(\gdmean - \rsmean) &\geq \frac{\eta^2(n-1)(1-\delta)s_{\min}}{2n} - \sum_{m=3}^n \eta^m (s_{\max}^m + s_{\max} \delta^{m-1} n^{m/2}),\\
%     \label{eq:proof-sketch-3}
%     \lambda_{\min}(\rsmean) &\geq 1 - \eta s_{\max} - s_{\max} \sum_{m=2}^n \eta^m \delta^{m-1} n^{m/2}.
% \end{align}
Also, the RHSs of \eqref{eq:proof-sketch-2} and \eqref{eq:proof-sketch-3} are nonnegative for $n \geq 4$, $\delta = n^{-1/2}$, $s_{\min} = n^{-1/4}$, $s_{\max} = n^{1/4}$, and $\eta \in [0, \frac{1}{6n}]$.
\end{lemma}
In a similar way as Lemma~\ref{lem:ssprobdep}, Lemma~\ref{lem:amgm} is proven by arranging $\rsmean$ and $\gdmean$ in increasing order of degree of $\eta$. These two matrices only differ in degrees at least 2. For the remaining steps, the crux of the proof is to carefully bound the degree-2 and degree-3 terms in $\gdmean-\rsmean$. Once \eqref{eq:proof-sketch-2} and \eqref{eq:proof-sketch-3} are shown, substituting the values of $\delta$, $s_{\min}$, and $s_{\max}$ assumed in Theorem~\ref{thm:rs-1} and a bit of elementary calculations finish the proof.

%% file: A_counterex_arxiv.tex
\section{Counterexamples for the first inequality, when $\eta_{n,K} = 1$}
\label{sec:finding-counterex}

In this section, we show that if we set $\eta_{n,K} = 1$, i.e., allow general PSD matrices, then counterexamples for the first inequality $\norm{\ssmean} \leq \norm{\rsmean}$ of Conjecture~\ref{conj:main} exist for $n\geq3$, $K\geq2$, $d \geq 2$.
This shows $\eta_{n,K} < 1$ is necessary when $n \geq 3$.
Using this counterexample, we will also demonstrate that at least for $d=2$, $\eta_{n,K}$ must also decrease with $K$ for the inequality to hold.
At the end of this section, we briefly describe a simple random procedure that we used to obtain such counterexamples, and report settings where we could/couldn't generate counterexamples to the conjecture.
In this section, we will occasionally omit subscripts in $\eta_{n,K}$ for simplicity.

\paragraph{Positive semidefinite counterexamples to $\norm{\ssmean} \leq \norm{\rsmean}$.}
We start by noting the following observation on ``extending'' a given counterexample to higher $n$ or $d$.
For $\norm{\ssmean} \leq \norm{\rsmean}$, once a counterexample $(\mA_1, \dots, \mA_n)$ is found for a tuple $(n,K,d,\eta)$, then $(\mA_1, \dots, \mA_n, \mI, \dots, \mI)$ is also a counterexample for $(n',K,d,\eta)$ where $n' > n$.
Also, for the special case of $\eta = 1$, a counterexample $(\mA_1, \dots, \mA_n)$ found for $(n,K,d,1)$ can be used to extend it to higher dimension $d'>d$ by letting $\wt{\mA}_i = \mA_i \oplus \zeros$, i.e., a direct sum of $\mA_i$ and a $(d'-d)$-by-$(d'-d)$ zero matrix.
Therefore, as long as we have $\eta = 1$, it suffices to present a counterexample for $n=3$, $K \geq 2$, and $d = 2$.

We will now show that the following three positive semidefinite matrices give a counterexample to $\norm{\ssmean} \leq \norm{\rsmean}$:
\begin{equation}
\label{eq:cntrex}
    \mA_1 = \begin{bmatrix}
    1/4 & \sqrt{3}/4\\
    \sqrt{3}/4 & 3/4
    \end{bmatrix},~~
    \mA_2 = \begin{bmatrix}
    1/4 & -\sqrt{3}/4\\
    -\sqrt{3}/4 & 3/4
    \end{bmatrix},~~
    \mA_3 = \begin{bmatrix}
    1 & 0\\
    0 & 0
    \end{bmatrix}.
\end{equation}
It is easy to check that $\mA_1$, $\mA_2$, and $\mA_3$ are all rank-1, and the powers of their without-replacement products read
\begin{align*}
    &(\mA_1 \mA_2 \mA_3)^K = \left(-\frac{1}{8}\right)^K 
    \begin{bmatrix}
    1 & 0\\\sqrt{3} & 0
    \end{bmatrix},~~~~~
    (\mA_1 \mA_3 \mA_2)^K = \left(-\frac{1}{8}\right)^K 
    \begin{bmatrix}
    -1/2 & \sqrt{3}/2\\-\sqrt{3}/2 & 3/2
    \end{bmatrix},\\
    &(\mA_2 \mA_1 \mA_3)^K = \left(-\frac{1}{8}\right)^K 
    \begin{bmatrix}
    1 & 0\\-\sqrt{3} & 0
    \end{bmatrix},~~
    (\mA_2 \mA_3 \mA_1)^K = \left(-\frac{1}{8}\right)^K 
    \begin{bmatrix}
    -1/2 & -\sqrt{3}/2\\\sqrt{3}/2 & 3/2
    \end{bmatrix},\\
    &(\mA_3 \mA_1 \mA_2)^K = \left(-\frac{1}{8}\right)^K 
    \begin{bmatrix}
    1 & -\sqrt{3}\\0 & 0
    \end{bmatrix},~~
    (\mA_3 \mA_2 \mA_1)^K = \left(-\frac{1}{8}\right)^K 
    \begin{bmatrix}
    1 & \sqrt{3}\\0 & 0
    \end{bmatrix}.
\end{align*}
Hence, the mean of the powers evaluates to
\begin{equation*}
    \ssmean(3,K,\mA_{1:3})
    = \frac{1}{6} \sum_{\sigma \in \mc S_3} \left (\prod_{i=1}^3 \mA_{\sigma(i)} \right )^K
    = \left(-\frac{1}{8}\right)^K
    \begin{bmatrix}
    1/2 & 0 \\ 0 & 1/2
    \end{bmatrix},
\end{equation*}
thus yielding
\begin{equation*}
    \norm{\ssmean(3,K,\mA_{1:3})} = \frac{1}{2\cdot 8^K}
    >
    \norm{\rsmean(3,K,\mA_{1:3})} = \norm{\ssmean(3,1,\mA_{1:3})}^K 
    = \frac{1}{16^K}.
\end{equation*}

\begin{figure}[tbp]
    \centering
    \includegraphics[width=0.6\linewidth]{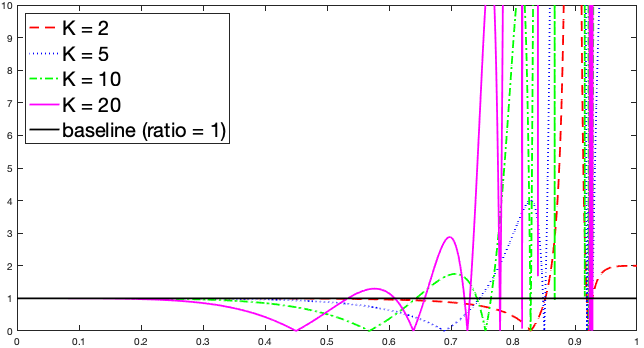}
    \caption{The ratio of $\norms{\ssmean}/\norms{\rsmean}$ calculated with $\wt{\mA}_i = (1-\eta)\mI + \eta\mA_i$, for $\eta \in [0,1]$.}
    \label{fig:counterex}
\end{figure}

\paragraph{Step-size constant $\eta_{n,K}$ should decay with $K$.}
We now further investigate this counterexample~\eqref{eq:cntrex} to (although rather computationally) demonstrate that for $d=2$, the step-size constant $\eta_{n,K}$ must decrease with $K$ for Conjecture~\ref{conj:main} to hold.
This starts from a natural question: 
given the counterexample~\eqref{eq:cntrex} for $\eta = 1$, what if we define positive definite matrices $\wt{\mA}_i = (1-\eta)\mI + \eta\mA_i$, for $\eta < 1$?

For the new matrices $\wt{\mA}_i$, $\rsmean(3,K,\wt{\mA}_{1:3})$ can be exactly evaluated through some calculations:
\begin{equation*}
    \rsmean(3,K,\wt{\mA}_{1:3})
    = \left ( \frac{1}{6} \sum_{\sigma \in \mc S_3} \prod_{i=1}^3 \wt{\mA}_{\sigma(i)} \right )^K
    = \left (1-\frac{3\eta}{2}+\frac{3\eta^2}{8}+\frac{\eta^3}{16} \right )^K \mI.
\end{equation*}
The polynomial $1-\frac{3\eta}{2}+\frac{3\eta^2}{8}+\frac{\eta^3}{16}$ has a root at $\eta = -4+2\sqrt{6}$. The other roots are outside $[0,1]$. The root $\eta = -4+2\sqrt{6}$ is the point where $\sum_{\sigma \in \mc S_3} \prod_{i=1}^3 \wt{\mA}_{\sigma(i)}$ turns from positive semidefinite to negative semidefinite, as $\eta$ increases from $0$ to $1$.

In case of $\ssmean(3,K,\wt{\mA}_{1:3})$, calculations are much more complicated and it is difficult to obtain a closed-form equation for general $K$'s. Instead, to see how the ratio between $\norm{\ssmean}$ and $\norm{\rsmean}$ changes with different $\eta$ and $K$, we plot the ratio $\norms{\ssmean(3,K,\wt{\mA}_{1:3})}/\norms{\rsmean(3,K,\wt{\mA}_{1:3})}$ over $\eta \in [0,1]$ for different $K$'s in Figure~\ref{fig:counterex}.
Notice that we also plot a baseline at $\text{ratio}=1$, so that if the ratio is above this line, the inequality $\norm{\ssmean}\leq\norm{\rsmean}$ is violated.
We observe that the ratio peaks to infinity at $\eta = -4+2\sqrt{6} \approx 0.899$ for all $K$'s, because $\rsmean = \zeros$ at $\eta = -4+2\sqrt{6}$. 
We also observe that as $K$ increases, the smallest value of $\eta_{n,K}$ where the graph of the ratio $\norm{\ssmean}/\norm{\rsmean}$ crosses the baseline decreases. This suggests that $\eta_{n,K}$ must decrease with $K$ for Conjecture~\ref{conj:main} to hold.

For these particular matrices $\wt{\mA}_{i}$, our calculations confirmed that choosing $\eta = \frac{1}{K^{1/3}}$ suffices to make the ratio $\norms{\ssmean(3,K,\wt{\mA}_{1:3})}/\norms{\rsmean(3,K,\wt{\mA}_{1:3})} \leq 0.999999$, for $K$ up to $200$. This is consistent with our belief that Conjecture~\ref{conj:main} is likely true for some choice of $\eta_{n,K}$ that decrease with~$K$.

\paragraph{Generating counterexamples via random samples.}
The counterexample found in \eqref{eq:cntrex} is in fact a refined version of counterexamples randomly generated from the following simple procedure.
For fixed $n \geq 2$, $K \geq 2$, $d \geq 2$, and $\eta \in (0,1]$, we used MATLAB to run the following random procedure.
\begin{itemize}
  \item[] Repeat the following steps $10^5$ times:
  \item[1.] Generate random matrices $\mU_1, \dots, \mU_n \in \reals^{d \times d}$ whose entries are sampled i.i.d.\ from standard Gaussian distribution: $[\mU_i]_{j,k} \sim \mc N(0,1)$ for all $i \in [n]$ and $j, k\in [d]$.
  \item[2.] Create unit-norm symmetric positive semidefinite matrices $\mM_1, \dots, \mM_n$ by $\mM_i = \frac{\mU_i \mU_i^T}{\norm{\mU_i \mU_i^T}}$.
  \item[3.] Let $\mA_i = (1-\eta)\mI + \eta \mM_i$ for all $i \in [n]$. 
  \item[4.] Test if $\norm{\ssmean(n,K,\mA_{1:n})} \leq  \norm{\rsmean(n,K,\mA_{1:n})} \leq \norm{\gdmean(n,K,\mA_{1:n})}$ holds.
\end{itemize}
Of course, such a random procedure does not prove the correctness of an inequality, but it can quickly disprove the inequality if a counterexample can be found.

After running the procedure for all possible combinations of 
\begin{align*}
    n &\in \{2,3,4,5,6\},~~
    K \in \{2,3,5,10,15,20,25,30\},\\
    d &\in \{2,5\},\quad\quad\quad~~~
    \eta \in \{0.25, 0.5, 0.75, 1\},
\end{align*}
we report the cases where the procedure succeeded in finding a counterexample to $\norm{\ssmean}\leq\norm{\rsmean}$. None of the runs generated a counterexample to $\norm{\rsmean}\leq\norm{\gdmean}$.
\begin{itemize}
    \item If $\eta = 1$, succeeded for all $n \geq 3$, $K$, and $d$.
    \item If $(\eta,n) = (0.75,3)$, succeeded for $K \geq 10$ when $d=2$, $K \geq 25$ when $d=5$.
    \item If $(\eta,n) = (0.75,4)$, succeeded for $K \geq 5$ when $d=2$, $K \geq 10$ when $d=5$.
    \item If $(\eta,n) = (0.75,5)$, succeeded for $K \geq 2$ when $d=2$, $K \geq 10$ when $d=5$.
    \item If $(\eta,n) = (0.75,6)$, succeeded for $K \geq 2$ when $d=2$, $K \geq 10$ when $d=5$.
    \item If $\eta = 0.5$ and $n \in \{3,4,5\}$, succeeded for $K = 30$ when $d=2$.
    \item If $\eta = 0.5$ and $n = 6$, succeeded for $K \geq 25$ when $d=2$.
\end{itemize}
In case of $n = 2$ or $\eta = 0.25$, none of the runs generated a counterexample.
We observe that for a fixed $\eta$, increasing $n$ makes it easier for the procedure to find counterexamples in lower $K$.
Also, it seems that the counterexamples are easier to generate in lower $d$.

We also ran the procedure using $\eta$ that decays with the $n$ and $K$. Trying the aforementioned sets of $n$, $K$, and $d$ with 
\begin{equation*}
\eta \in \left \{\frac{1}{(nK)^{1/4}}, \frac{1}{(nK)^{1/3}}, \frac{1}{(nK)^{1/2}}, \frac{1}{nK} \right \}    
\end{equation*}
did not generate any counterexamples for both inequalities.

We note again that the inability of the random procedure to find counterexamples does not imply that the conjecture is true. Recall that random sampling procedures fail to generate counterexamples to $\norm{\rsmean} \leq \norm{\gdmean}$, but counterexamples consisting of positive semidefinite matrices were found by \citet{de2020random}.

%% file: B_otherver_arxiv.tex
\section{Other versions of the Recht-R\'e conjecture \eqref{eq:recht-re-conj} and Conjecture~\ref{conj:main}}
\label{sec:othervers}

In the main paper, we focused on the ``norm-of-expectation-of-epoch'' version of the Recht-R\'e conjecture~\eqref{eq:recht-re-conj}, which we rewrite below:
\begin{equation}
\label{eq:otherver-1}
    \norm{\frac{1}{n!} \sum_{\sigma \in \mc S_n} \prod_{i=1}^n \mA_{\sigma(i)}} 
    \leq
    \norm{\frac{1}{n} \sum_{i=1}^n \mA_i}^n.
\end{equation}
In this section, we will introduce other versions of the conjecture proposed in the same paper \citep{recht2012toward} and its commentary \citep{duchi2012commentary}. These variants immediately motivate their corresponding versions of our Conjecture~\ref{conj:main}. We state these alternative versions of our $\norm{\ssmean}\leq\norm{\rsmean}$ conjecture and prove a simple special case.

To this end, we introduce the notation of without-replacement expectation $\Ewo[\cdot]$ and with-replacement expectation $\Ewr[\cdot]$ from \citet{recht2012toward}.
Suppose we are given a domain $\sD$, a codomain $\sD'$, $n$ elements of the domain $\vx_1, \dots, \vx_n \in \sD$, and a function $f: \sD^m \to \sD'$, where $m \leq n$. The \emph{without-replacement expectation} of $f$ is defined as
\begin{equation*}
    \Ewo[f(\vx_{i_1}, \dots, \vx_{i_m})] \defeq \frac{(n-m)!}{n!} \sum_{\substack{i_1,\dots,i_m \in [n]\\i_1, \dots, i_m\text{ distinct }}} f(\vx_{i_1}, \dots, \vx_{i_m}),
\end{equation*}
and \emph{with-replacement expectation} of $f$ is defined as
\begin{equation*}
    \Ewr[f(\vx_{i_1}, \dots, \vx_{i_m})] \defeq \frac{1}{n^m} \sum_{\substack{(i_1,\dots,i_m) \in [n]^m}} f(\vx_{i_1}, \dots, \vx_{i_m}).
\end{equation*}
Using this new notation, we can rewrite \eqref{eq:otherver-1} as 
\begin{equation}
\label{eq:otherver-2}
    \norm{\Ewo\left[ \prod\nolimits_{j=1}^n \mA_{i_j} \right]}
    \leq
    \norm{\Ewr\left[ \prod\nolimits_{j=1}^n \mA_{i_j} \right]}.
\end{equation}

We will now present other versions of the conjecture \citep{recht2012toward, duchi2012commentary}, including the symmetrized and expectation-of-norm versions.
The conjectures say: for any $n$ positive semidefinite matrices $\mA_1, \dots, \mA_n \in \sS^d_+$ and $m \leq n$,
\begin{align}
    \label{eq:otherver-3}
    \norm{\Ewo\left[ \prod\nolimits_{j=1}^m \mA_{i_j} \right]}
    &\leq
    \norm{\Ewr\left[ \prod\nolimits_{j=1}^m \mA_{i_j} \right]}\\
    \label{eq:otherver-4}
    \norm{\Ewo\left[ \prod\nolimits_{j=m}^1 \mA_{i_j}  \prod\nolimits_{j=1}^m \mA_{i_j} \right]}
    &\leq
    \norm{\Ewr\left[ \prod\nolimits_{j=m}^1 \mA_{i_j} \prod\nolimits_{j=1}^m \mA_{i_j} \right]}\\
    \label{eq:otherver-5}
    \Ewo\left[ \norm{\prod\nolimits_{j=1}^m \mA_{i_j}} \right]
    &\leq
    \Ewr\left[ \norm{\prod\nolimits_{j=1}^m \mA_{i_j}} \right]\\
    \label{eq:otherver-6}
    \Ewo\left[ \norm{\prod\nolimits_{j=m}^1 \mA_{i_j} \prod\nolimits_{j=1}^m \mA_{i_j}} \right]
    &\leq
    \Ewr\left[ \norm{\prod\nolimits_{j=m}^1 \mA_{i_j} \prod\nolimits_{j=1}^m \mA_{i_j}} \right]
\end{align}
The inequality \eqref{eq:otherver-3} is the one introduced in the main paper, but note that we multiply over $m$ matrices rather than $n$ matrices; the version we studied in the main paper is \eqref{eq:otherver-3} with $m = n$.
The next inequality \eqref{eq:otherver-4} is the \emph{symmetrized} version of \eqref{eq:otherver-3}, which is relevant to the convergence analysis of stochastic gradient methods when we want to bound the expectation of the squared $\ell_2$ distance to global optima (see e.g., the proof of Theorem~2 in \citet{ahn2020sgd}).
The inequalities~\eqref{eq:otherver-5} and \eqref{eq:otherver-6} are the \emph{expectation-of-norm} versions proposed by \citet{duchi2012commentary}, as opposed to the norm-of-expectation versions in the original paper \citep{recht2012toward}. As pointed out in \citet{duchi2012commentary}, these versions arise when we try to bound the expectation of (not squared) $\ell_2$ distance to global minima.

Several papers study these variants. 
It was introduced in the main paper that \citet{lai2020recht} prove \eqref{eq:otherver-3} for the $n = 3$ case because we restricted our attention to $m = n$ for simplicity, but in fact, the authors prove \eqref{eq:otherver-3} for $m = 2$ or $3$ and general $n \geq m$.
\citet{israel2016arithmetic} prove \eqref{eq:otherver-5} for $m \leq 3$ and general $n \geq m$. \citet{albar2017noncommutative,albar2018symmetrized} study a weaker version of \eqref{eq:otherver-3} and \eqref{eq:otherver-4}, where the RHSs are scaled up by a multiplicative factor.
As also mentioned in the main text, \citet{de2020random} disproves the $m = n$ version of \eqref{eq:otherver-3} by constructing counterexamples. The paper also provides counterexamples to the $m = n$ versions of \eqref{eq:otherver-4} and \eqref{eq:otherver-5}, suggesting that the conjectures are not likely true for general positive semidefinite matrices.

In light of the four different versions of the Recht-R\'e conjecture, one can also naturally define the corresponding versions of Conjecture~\ref{conj:main}. Focusing on the $\norm{\ssmean} \leq \norm{\rsmean}$ inequality, we can define three additional variants of the conjecture. 
\begin{conjecture}
For any $n \geq 2$, $K \geq 1$, and $d \geq 1$, there exists a step-size constant $\eta_{n,K} \in (0,1]$ such that the following statement holds:
Suppose $\mA_1, \dots, \mA_n \in \sS^d_{+}$ 
%are symmetric positive definite matrices satisfying
satisfy
$(1-\eta_{n,K}) \mI \preceq \mA_i \preceq \mI$ for all $i \in [n]$. 
Then, the following inequalities hold:
\begin{align}
    \label{eq:otherver-7}
    \norm{
    \E_\sigma \!\left[ 
    \bigg (\prod_{i=1}^n \mA_{\sigma(i)} \bigg )^K
    \right]
    }
    &\leq
    \norm{
    \E_{\sigma_{1:K}} \!\left[ 
    \prod_{k=1}^K \prod_{i=1}^n \mA_{\sigma_k(i)} 
    \right]
    }\\
    \label{eq:otherver-8}
    \norm{
    \E_\sigma \!\left[ 
    \bigg (\prod_{i=n}^1 \mA_{\sigma(i)} \bigg )^K
    \bigg (\prod_{i=1}^n \mA_{\sigma(i)} \bigg )^K
    \right]
    }
    &\leq
    \norm{
    \E_{\sigma_{1:K}} \!\left[
    \bigg (\prod_{k=K}^1 \prod_{i=n}^1 \mA_{\sigma_k(i)} \bigg)
    \bigg (\prod_{k=1}^K \prod_{i=1}^n \mA_{\sigma_k(i)} \bigg)
    \right]
    }\\
    \label{eq:otherver-9}
    \E_\sigma \!\left[ 
    \norm{\bigg (\prod_{i=1}^n \mA_{\sigma(i)} \bigg )^K}
    \right]
    &\leq
    \E_{\sigma_{1:K}} \!\left[ 
    \norm{\prod_{k=1}^K \prod_{i=1}^n \mA_{\sigma_k(i)}}
    \right]
    \\
    \label{eq:otherver-10}
    \E_\sigma \!\left[ 
    \norm{
    \bigg (\prod_{i=n}^1 \mA_{\sigma(i)} \bigg )^K
    \bigg (\prod_{i=1}^n \mA_{\sigma(i)} \bigg )^K
    }
    \right]
    &\leq
    \E_{\sigma_{1:K}} \!\left[
    \norm{
    \bigg (\prod_{k=K}^1 \prod_{i=n}^1 \mA_{\sigma_k(i)} \bigg)
    \bigg (\prod_{k=1}^K \prod_{i=1}^n \mA_{\sigma_k(i)} \bigg)
    }
    \right]
\end{align}
\end{conjecture}
Here, $\E_{\sigma_{1:K}}$ denotes the expectation over i.i.d.\ random permutations $\sigma_1, \dots, \sigma_K \sim {\rm Unif}(\mc S_n)$.
One can check that \eqref{eq:otherver-7} is the $\norm{\ssmean}\leq\norm{\rsmean}$ inequality in Conjecture~\ref{conj:main}. The new conjecture in \eqref{eq:otherver-8} is the symmetrized version, and the remaining two \eqref{eq:otherver-9} and \eqref{eq:otherver-10} are the expectation-of-norm versions of \eqref{eq:otherver-7} and \eqref{eq:otherver-8}.

We can prove a simple special case of \eqref{eq:otherver-9} and \eqref{eq:otherver-10}, for general PSD matrices:
\begin{proposition}
For $n = 2$, $K = 2$, and $\eta_{2,2} = 1$, the conjectures \eqref{eq:otherver-9} and \eqref{eq:otherver-10} are true.
\end{proposition}
\begin{proof}
~~For $n = 2$ and $K = 2$, \eqref{eq:otherver-9} and \eqref{eq:otherver-10} boil down to the following: for $\mA, \mB \in \sS^d_+$,
\begin{align*}
    &\frac{1}{2}(\norm{\mA \mB \mA \mB} + \norm{\mB \mA \mB \mA})\\
    &~\leq
    \frac{1}{4}(\norm{\mA \mB \mA \mB} + \norm{\mA \mB \mB \mA} + \norm{\mB \mA \mA \mB} + \norm{\mB \mA \mB \mA}),\text{ and }\\
    &\frac{1}{2}(\norm{\mB \mA \mB \mA \mA \mB \mA \mB} + \norm{\mA \mB \mA \mB \mB \mA \mB \mA})\\
    &~\leq
    \frac{1}{4}
    (\norm{\mB \mA \mB \mA \mA \mB \mA \mB} 
    +\norm{\mA \mB \mB \mA \mA \mB \mB \mA} 
    +\norm{\mB \mA \mA \mB \mB \mA \mA \mB}
    +\norm{\mA \mB \mA \mB \mB \mA \mB \mA}).
\end{align*}
Recalling $\norm{\mC} = \norm{\mC^T}$ and $\norm{\mC}^2 = \norm{\mC \mC^T} = \norm{\mC^T \mC}$ for any matrix $\mC$, it suffices to prove 
\begin{align*}
    \norm{\mA \mB \mA \mB} = \norm{\mB \mA \mB \mA}
    \leq \norm{\mA \mB \mB \mA} = \norm{\mB \mA \mA \mB}.
\end{align*}
The inequality above is proven from the following well-known inequality \citep{bhatia1993more,horn1995norm}: for any $\mX \in \reals^{d \times d}$, $\mY \in \reals^{k \times k}$, and $\mZ \in \reals^{d \times k}$, 
\begin{equation}
\label{eq:otherver-11}
    2\norms{\mX^T \mZ \mY} \leq \norms{\mX \mX^T \mZ + \mZ \mY \mY^T}.
\end{equation}
Using \eqref{eq:otherver-11}, the proof of $\norm{\mA\mB\mA\mB} \leq \norm{\mA\mB\mB\mA}$ can be obtained as follows:
\begin{align*}
    \norm{\mA\mB\mA\mB}^2
    =&\,\norm{\mA\mB\mA\mB\mB\mA\mB\mA}\\
    \leqwtxt{(a)} &\,
    \bhalf (\norm{\mB\mA\mA\mB\mA\mB\mB\mA + \mA\mB\mB\mA\mB\mA\mA\mB})\\
    \leqwtxt{(b)} &\, \norm{\mA\mB\mB\mA\mB\mA\mA\mB}\\
    \leqwtxt{(c)} &\,\bhalf (\norm{\mA\mB\mB\mA\mA\mB\mB\mA + \mB\mA\mA\mB\mB\mA\mA\mB})\\
    \leqwtxt{(d)} &\,\norm{\mA\mB\mB\mA\mA\mB\mB\mA}
    = \norm{\mA\mB\mB\mA}^2.
\end{align*}
Here, (a) applies \eqref{eq:otherver-11} with $(\mX,\mY,\mZ) = (\mB\mA,\mB\mA,\mA\mB\mB\mA)$, followed by a triangle inequality and $\norm{\mC} = \norm{\mC^T}$ in (b). The next step (c) uses \eqref{eq:otherver-11} with $(\mX,\mY,\mZ) = (\mA\mB\mB\mA,\mB\mA\mA\mB,\mI)$, and then a triangle inequality and $\norm{\mC\mC^T\mC\mC^T} = \norm{\mC^T\mC\mC^T\mC}$ in (d).
\end{proof}

%% file: C_pflemmas_arxiv.tex
\section{Proofs of technical lemmas}
\subsection{Proof of Lemma~\ref{lem:ssprobdep}}
\label{sec:proof-lem-ssprobdep}
For any fixed set of $n$ symmetric matrices $\mM_1, \dots, \mM_n \in \reals^{d \times d}$ and permutation $\sigma \in \mc S_n$, we define the following matrix polynomials, for $m = 0 , \dots, n$:
\begin{align}
\label{eq:elem-sym-poly}
    e_m(\sigma) \defeq \sum_{1 \leq i_1 < i_2 < \cdots < i_m \leq n} \mM_{\sigma(i_1)} \mM_{\sigma(i_2)} \cdots \mM_{\sigma(i_m)}.
\end{align}
Note that this is a noncommutative version of the elementary symmetric polynomials. Notice also that $e_0(\sigma) = \mI$ and $e_1(\sigma) = \sum_{i=1}^n \mM_i$ are independent of the permutation $\sigma$, so we will write $e_1 \defeq \sum_{i=1}^n \mM_i$ for simplicity.
Using these polynomials, one can write the product of $\mA_i \defeq \mI - \eta \mM_i$ according to the order $\sigma$ as
\begin{align}
\label{eq:proof-lem-ssprobdep-1}
    \prod_{i=1}^n \mA_{\sigma(i)} = \prod_{i=1}^n (\mI - \eta \mM_{\sigma(i)})
    = \sum_{m=0}^n (-\eta)^m e_m(\sigma)
    = \mI - \eta e_1 + \sum_{m=2}^n (-\eta)^m e_m(\sigma).
\end{align}
Now, denoting the mean of $e_m(\sigma)$ over all $\sigma \in \mc S_n$ as $\overline{e_m} \defeq \frac{1}{n!}\sum_{\sigma \in \mc S_n} e_m(\sigma)$, the mean of the product can be similarly written as
\begin{align}
    \label{eq:proof-lem-ssprobdep-2}
    \frac{1}{n!} \sum_{\sigma \in \mc S_n} \prod_{i=1}^n \mA_{\sigma(i)} 
    = \mI - \eta e_1 + \sum_{m=2}^n (-\eta)^m \overline{e_m}.
\end{align}
We will use \eqref{eq:proof-lem-ssprobdep-1} and \eqref{eq:proof-lem-ssprobdep-2} to calculate $\rsmean$ and $\ssmean$, and then compute their difference.

First, recall that the $K$-th power of \eqref{eq:proof-lem-ssprobdep-2} is $\rsmean$. From \eqref{eq:proof-lem-ssprobdep-2}, $\rsmean$ can be written as
\begin{align}
    \rsmean
    =&\,\mI - \eta K e_1 
    + \eta^2 \bigg ( 
    \binom{K}{1} \overline{e_2} 
    + \binom{K}{2} e_1^2 
    \bigg )\notag\\
    & - \eta^3 \bigg ( 
    \indic_{n \geq 3} \binom{K}{1} \overline{e_3} 
    + \binom{K}{2} (e_1 \overline{e_2} + \overline{e_2} e_1) 
    + \indic_{K \geq 3}  \binom{K}{3} e_1^3 
    \bigg )\notag\\
    & + \eta^4 \bigg ( 
    \indic_{n \geq 4} \binom{K}{1} \overline{e_4} 
    + \binom{K}{2} \left (\overline{e_2}^2 + \indic_{n \geq 3} (e_1 \overline{e_3} + \overline{e_3}e_1) \right )\notag\\
    &\quad\qquad+ \indic_{K \geq 3} \binom{K}{3} \left ( e_1^2 \overline{e_2} + e_1 \overline{e_2} e_1 + \overline{e_2} e_1^2 \right )
    + \indic_{K \geq 4} \binom{K}{4} e_1^4
    \bigg )\notag\\
    &+ \sum_{m=5}^{nK} (-\eta)^m p_m,\label{eq:proof-lem-ssprobdep-3}
\end{align}
where $\indic_{E}$ is the 0-1 indicator function of an event $E$, and each $p_m$ is some appropriately defined noncommutative polynomial of $\mM_{1}, \dots, \mM_{n}$ whose monomials are of degree $m$.

Similarly, for a fixed $\sigma \in \mc S_n$, the $K$-th power of \eqref{eq:proof-lem-ssprobdep-1} reads
\begin{align*}
    \left(\prod_{i=1}^n \mA_{\sigma(i)}\right)^K
    =&\,\mI - \eta K e_1 
    + \eta^2 \bigg (
    \binom{K}{1} e_2(\sigma)
    + \binom{K}{2} e_1^2 
    \bigg )\\
    & - \eta^3 \bigg ( 
    \indic_{n \geq 3} \binom{K}{1} e_3(\sigma)
    + \binom{K}{2} (e_1 e_2(\sigma) + e_2(\sigma) e_1)
    + \indic_{K \geq 3}  \binom{K}{3} e_1^3
    \bigg )\\
    & + \eta^4 \bigg ( 
    \indic_{n \geq 4} \binom{K}{1} e_4(\sigma) 
    + \binom{K}{2} \left (e_2(\sigma)^2 + \indic_{n \geq 3} (e_1 e_3(\sigma) + e_3(\sigma)e_1) \right )\\
    &\quad\qquad+ \indic_{K \geq 3} \binom{K}{3} \left ( e_1^2 e_2(\sigma) + e_1 e_2(\sigma) e_1 + e_2(\sigma) e_1^2 \right )
    + \indic_{K \geq 4} \binom{K}{4} e_1^4
    \bigg )\\
    &+ \sum_{m=5}^{nK} (-\eta)^m q_m(\sigma),
\end{align*}
where each $q_m(\sigma)$ is some polynomial of $\mM_{\sigma(1)}, \dots, \mM_{\sigma(n)}$ with monomials of degree $m$. Taking the mean over all $\sigma \in \mc S_n$, we get
\begin{align}
    \ssmean 
    =&\,\mI - \eta K e_1 
    + \eta^2 \bigg (
    \binom{K}{1} \overline{e_2}
    + \binom{K}{2} e_1^2 
    \bigg )\notag\\
    & - \eta^3 \bigg ( 
    \indic_{n \geq 3} \binom{K}{1} \overline{e_3}
    + \binom{K}{2} (e_1 \overline{e_2} + \overline{e_2} e_1)
    + \indic_{K \geq 3}  \binom{K}{3} e_1^3
    \bigg )\notag\\
    & + \eta^4 \bigg ( 
    \indic_{n \geq 4} \binom{K}{1} \overline{e_4}
    + \binom{K}{2} \bigg (\frac{1}{n!} \sum_{\sigma\in\mc S_n} e_2(\sigma)^2 + \indic_{n \geq 3} (e_1 \overline{e_3} + \overline{e_3}e_1) \bigg )\notag\\
    &\quad\qquad+ \indic_{K \geq 3} \binom{K}{3} \left ( e_1^2 \overline{e_2} + e_1 \overline{e_2} e_1 + \overline{e_2} e_1^2 \right )
    + \indic_{K \geq 4} \binom{K}{4} e_1^4
    \bigg )\notag\\
    &+ \sum_{m=5}^{nK} (-\eta)^m \frac{1}{n!} \sum_{\sigma \in \mc S_n} q_m(\sigma).\label{eq:proof-lem-ssprobdep-4}
\end{align}

Comparing \eqref{eq:proof-lem-ssprobdep-3} and \eqref{eq:proof-lem-ssprobdep-4}, we can see that $\rsmean$ and $\ssmean$ have identical terms up to degree 3. The difference of the two matrices can be written as
\begin{align}
\label{eq:proof-lem-ssprobdep-5}
\rsmean - \ssmean 
= \frac{\eta^4 K (K-1)}{2} \left ( \left(\frac{\sum_{\sigma \in \mc S_n} e_2(\sigma)}{n!}\right)^2 - \frac{\sum_{\sigma \in \mc S_n} e_2(\sigma)^2}{n!}\right) + O(\eta^5).
\end{align}
If we take a close look at the fourth-order term in \eqref{eq:proof-lem-ssprobdep-5}, we can notice that
\begin{align}
    \left(\frac{\sum_{\sigma \in \mc S_n} e_2(\sigma)}{n!}\right)^2 - \frac{\sum_{\sigma \in \mc S_n} e_2(\sigma)^2}{n!}
    &= -\frac{1}{(n!)^2} \left (n! \sum_{\sigma \in \mc S_n} e_2(\sigma)^2 - \left (\sum_{\sigma \in \mc S_n} e_2(\sigma) \right)^2 \right)\notag\\
    &= -\frac{1}{2(n!)^2} \sum_{\sigma, \tilde \sigma \in \mc S_n} (e_2(\sigma) - e_2(\tilde \sigma))^2. \label{eq:proof-lem-ssprobdep-6}
\end{align}

Next, we show that for any $\sigma, \tilde \sigma \in \mc S_n$, $e_2(\sigma) - e_2(\tilde \sigma)$ is a skew-symmetrifc matrix, which implies that each $(e_2(\sigma) - e_2(\tilde \sigma))^2$ is a negative semidefinite matrix.
To illustrate why, we first give some examples. Consider $n = 3$ and $(\sigma(1), \sigma(2), \sigma(3)) = (1,2,3)$. If $(\tilde \sigma(1), \tilde \sigma(2), \tilde \sigma(3)) = (2,1,3)$, then $e_2(\sigma) - e_2(\tilde \sigma) = \mM_1 \mM_2 - \mM_2 \mM_1$, which is a skew-symmetric matrix. Similarly, if $(\tilde \sigma(1), \tilde \sigma(2), \tilde \sigma(3)) = (3,2,1)$, then $e_2(\sigma) - e_2(\tilde \sigma) = \mM_1 \mM_2 - \mM_2 \mM_1 + \mM_1 \mM_3 - \mM_3 \mM_1 + \mM_2 \mM_3 - \mM_3 \mM_2$ is also skew-symmetric. More formally, for any $\sigma, \tilde \sigma \in \mc S_n$, we have
\begin{equation*}
    e_2(\sigma) - e_2(\tilde \sigma) = \sum_{\substack{i, j \in [n], i \neq j,\\\sigma^{-1}(i) < \sigma^{-1}(j),\\\tilde \sigma^{-1}(i) > \tilde \sigma^{-1}(j)}} \mM_i \mM_j - \mM_j \mM_i,
\end{equation*}
which is hence skew-symmetric.

Because each $(e_2(\sigma) - e_2(\tilde \sigma))^2$ is negative semidefinite, for any subset $\mc T \subseteq \mc S_n \times \mc S_n$, we have
\begin{equation*}
    -\sum_{(\sigma, \tilde \sigma) \in \mc S_n \times \mc S_n} (e_2(\sigma) - e_2(\tilde \sigma))^2 \succeq -\sum_{(\sigma, \tilde \sigma) \in \mc T} (e_2(\sigma) - e_2(\tilde \sigma))^2.
\end{equation*}
We will choose $\mc T$ to be the set of pairs $(\sigma, \tilde \sigma)$ such that $\tilde \sigma$ is obtained from $\sigma$ by ``flipping an adjacent pair of inputs'':
\begin{align*}
    \mc T \defeq \big \{ (\sigma, \tilde \sigma) \in \mc S_n \times \mc S_n : 
    \exists j \in [n-1] ~\text{such that}&~ \sigma(j) = \tilde \sigma(j+1), \sigma(j+1) = \tilde \sigma(j), \\
    &~\text{and}~ \forall i \in [n] \setminus \{j, j+1\}, \sigma(i) = \tilde \sigma(i)
    \big \}.
\end{align*}
Elementary calculations reveal that for each $i, j \in [n]$ satisfying $i \neq j$, there are $(n-1)!$ pairs of $(\sigma, \tilde \sigma) \in \mc T$ such that $e_2(\sigma) - e_2(\tilde \sigma) = \mM_i \mM_j - \mM_j \mM_i$.
As a result, 
\begin{equation}
\label{eq:proof-lem-ssprobdep-7}
    -\sum_{\sigma, \sigma' \in \mc S_n} (e_2(\sigma) - e_2(\sigma'))^2
    \succeq -(n-1)! \sum_{i,j \in [n], i\neq j} (\mM_i \mM_j - \mM_j \mM_i)^2.
\end{equation}
Combining \eqref{eq:proof-lem-ssprobdep-5}, \eqref{eq:proof-lem-ssprobdep-6}, and \eqref{eq:proof-lem-ssprobdep-7} finishes the proof.

\subsection{Proof of Lemma~\ref{lem:sspsd}}
\label{sec:proof-lem-sspsd}
Since the positive definite matrices $\mA$ and $\mB$ satisfy $(1-\frac{1}{2K}) \mI \preceq \mA \preceq \mI$ and $(1-\frac{1}{2K}) \mI \preceq \mB \preceq \mI$, we can write them as
\begin{equation*}
    \mA = (1-\eta) \mI + \eta \tilde \mA ~~\text{and}~~
    \mB = (1-\eta) \mI + \eta \tilde \mB,
\end{equation*}
where $\eta \defeq \frac{1}{2K}$, $\zeros \preceq \tilde \mA \preceq \mI$ and $\zeros \preceq \tilde \mB \preceq \mI$. Using this representation, we can express $(\mA \mB)^K$ as
\begin{align*}
    (\mA \mB)^K &= \big (((1-\eta) \mI + \eta \tilde \mA)((1-\eta) \mI + \eta \tilde \mB) \big)^K\\
    &= ((1-\eta)^2 \mI + \eta(1-\eta)\tilde \mA + \eta(1-\eta) \tilde \mB + \eta^2 \tilde \mA \tilde \mB)^K\\
    &= (1-\eta)^{2K} \mI + K \eta (1-\eta)^{2K-1} (\tilde \mA + \tilde \mB) + \sum_{k=2}^{2K} \eta^k (1-\eta)^{2K-k} p_k(\tilde \mA, \tilde \mB),
\end{align*}
where $p_k$'s are some appropriately defined noncommutative polynomials of $\tilde \mA$ and $\tilde \mB$, whose monomials are of degree exactly $k$. Now compare this with the following expression of 1:
\begin{equation*}
    1^K = ((1-\eta)^2 + 2\eta(1-\eta) + \eta^2)^K
    = (1-\eta)^{2K} + 2K\eta(1-\eta)^{2K-1} + \sum_{k=2}^{2K} \eta^k(1-\eta)^{2K-k} q_k,
\end{equation*}
where $q_k$'s are some nonnegative integers.
Notice that the number of monomials in $p_k(\tilde \mA, \tilde \mB)$ is exactly $q_k$, and each of them have spectral norm bounded above by 1 (this is because $\norms{\tilde \mA} \leq 1$ and $\norms{\tilde \mB} \leq 1$). Therefore,
\begin{align*}
    &\norm{\sum_{k=2}^{2K} \eta^k (1-\eta)^{2K-k} p_k(\tilde \mA, \tilde \mB)}
    \leq \sum_{k=2}^{2K} \eta^k (1-\eta)^{2K-k} \norm{p_k(\tilde \mA, \tilde \mB)}\\
    \leq& \sum_{k=2}^{2K} \eta^k (1-\eta)^{2K-k} q_k
    = 1-(1-\eta)^{2K} - 2K\eta (1-\eta)^{2K-1}.
\end{align*}
From this observation, the matrix inequality that we want to show
\begin{align*}
    \zeros &\preceq (\mA \mB)^K + (\mB \mA)^K\\
    &= 2(1-\eta)^{2K} \mI + 2K \eta (1-\eta)^{2K-1} (\tilde \mA + \tilde \mB) + \sum_{k=2}^{2K} \eta^k (1-\eta)^{2K-k} \left(p_k(\tilde \mA, \tilde \mB) + p_k(\tilde \mB, \tilde \mA)\right)
\end{align*}
is true if
\begin{align}
\label{eq:proof-lem-sspsd-1}
    2(1-(1-\eta)^{2K} - 2K\eta (1-\eta)^{2K-1}) \leq 2(1-\eta)^{2K},
\end{align}
i.e., if the norm the sum of high-degree terms is not greater than that of $2(1-\eta)^{2K} \mI$.
The inequality~\eqref{eq:proof-lem-sspsd-1} is true for $\eta = \frac{1}{2K}$ because the sequences $s_1(K) = (1-\frac{1}{2K})^{2K}$ and $s_2(K) = (1-\frac{1}{2K})^{2K-1}$ are both monotone increasing in $K$, and $2s_1(1) + s_2(1) = 1$.

\subsection{Proof of Lemma~\ref{lem:numrad}}
\label{sec:proof-lem-numrad}
In this section, we use $\I$ to denote the imaginary number $\I$: $\I^2 = -1$.
We use the following equivalent definition of numerical radius:
\begin{equation}
\label{eq:proof-lem-numrad-1}
    w(\mM) 
    =
    \sup_{\theta \in \reals} \norm{\cos \theta \left(\frac{\mM + \mM^H}{2}\right) + \I \sin \theta \left(\frac{\mM - \mM^H}{2}\right)}.
\end{equation}
To see why this equivalent definition~\eqref{eq:proof-lem-numrad-1} holds, note that
\begin{align*}
    w(\mM) \defeq& \sup\,\{|\vv^H \mM \vv|: \vv \in \C^d, \norm{\vv} \leq 1 \}\\
    =& \sup\,\{\Re(e^{\I \theta} \vv^H \mM \vv): \theta \in \reals, \vv \in \C^d, \norm{\vv} \leq 1 \}\\
    =& \sup\,\left \{ \frac{\vv^H e^{\I \theta} \mM \vv + \vv^H e^{-\I \theta} \mM^H \vv}{2} : \theta \in \reals, \vv \in \C^d, \norm{\vv} \leq 1 \right \}\\
    =& \sup_{\theta \in \reals} w\left( \frac{e^{\I \theta} \mM + e^{-\I \theta} \mM^H}{2} \right)
    = \sup_{\theta \in \reals} \norm{ \frac{e^{\I \theta} \mM + e^{-\I \theta} \mM^H}{2} }\\
    =& \sup_{\theta \in \reals} \norm{\cos \theta \left(\frac{\mM + \mM^H}{2}\right) + \I \sin \theta \left(\frac{\mM - \mM^H}{2}\right)},
\end{align*}
where the second-to-last equality used the fact that for Hermitian matrices, the numerical radius is equal to the spectral norm (see e.g., \citet{dragomir2013inequalities}).

The proof will proceed as follows.
We will name the entries of $\mQ$ and $\mLambda$, and calculate the entries of $\mQ \mLambda \mQ^{-1}$. Substituting them into \eqref{eq:proof-lem-numrad-1}, we will show that the supremum is attained at $\theta = 0$ via direct calculations of the spectral norm.

First, we name the entries of $\mQ$ and $\mLambda$ as
\begin{equation*}
    \mQ = \begin{bmatrix} p & q\\ r & s\end{bmatrix},~~
    \mLambda = \begin{bmatrix} \alpha & 0 \\ 0 & \beta \end{bmatrix}.
\end{equation*}
By the assumptions, the entries are real and have to satisfy $ps - qr \neq 0$ and $\alpha, \beta \geq 0$.
Using these entries, we can then calculate $\half(\mQ \mLambda \mQ^{-1} + \mQ^{-T} \mLambda \mQ^{T})$ and $\half(\mQ \mLambda \mQ^{-1} - \mQ^{-T} \mLambda \mQ^{T})$ directly:
\begin{align}
    &\half(\mQ \mLambda \mQ^{-1} + \mQ^{-T} \mLambda \mQ^{T})\notag\\
    = &
    \frac{1}{2(ps-qr)} 
    \begin{bmatrix} p & q \\ r & s \end{bmatrix}
    \begin{bmatrix} \alpha & 0 \\ 0 & \beta \end{bmatrix}
    \begin{bmatrix} s & -q \\ -r & p \end{bmatrix}
    + 
    \frac{1}{2(ps-qr)} 
    \begin{bmatrix} s & -r \\ -q & p \end{bmatrix}
    \begin{bmatrix} \alpha & 0 \\ 0 & \beta \end{bmatrix}
    \begin{bmatrix} p & r \\ q & s \end{bmatrix}\notag\\
    = &
    \frac{1}{2(ps-qr)} 
    \begin{bmatrix} 
    2\alpha p s - 2\beta q r & (\beta-\alpha)(pq - rs) \\ 
    (\beta-\alpha)(pq - rs) & 2\beta p s - 2\alpha q r \end{bmatrix},\label{eq:proof-lem-numrad-4}\\
    &\half(\mQ \mLambda \mQ^{-1} - \mQ^{-T} \mLambda \mQ^{T})\notag\\
    = &
    \frac{1}{2(ps-qr)} 
    \begin{bmatrix} 
    0 & (\beta-\alpha)(pq + rs) \\ 
    -(\beta-\alpha)(pq + rs) & 0 
    \end{bmatrix}.\label{eq:proof-lem-numrad-5}
\end{align}
For 2-by-2 matrices, the norm on the RHS of \eqref{eq:proof-lem-numrad-1} can be directly calculated as a function of $\theta$. We prove and use the following lemma:
\begin{lemma}
For real scalars $a, b, c, d \in \reals$, we have
\label{lem:proof-lem-numrad-sub1}
\begin{align}
\label{eq:proof-lem-numrad-2}
2\norm{\cos \theta \begin{bmatrix}a & b\\b & c\end{bmatrix} + \I \sin \theta \begin{bmatrix}0 & d\\-d & 0\end{bmatrix}}^2
&= (a^2 + 2b^2+c^2) \cos^2\theta + 2 d^2 \sin^2\theta\\
&\hspace{-15pt}+ \sqrt{(a^2-c^2)^2 \cos^4 \theta + 4(a+c)^2 \cos^2\theta (b^2 \cos^2\theta + d^2 \sin^2\theta)}\notag.
\end{align}
\end{lemma}
\begin{proof}~~
The spectral norm squared in the LHS is equal to the maximum eigenvalue of 
\begin{align*}
    &\begin{bmatrix}
    a\cos \theta & b\cos \theta + \I d \sin \theta\\
    b\cos \theta - \I d \sin \theta& c\cos \theta
    \end{bmatrix}^2\\
    =&
    \begin{bmatrix}
    (a^2+b^2) \cos^2 \theta + d^2 \sin^2 \theta &
    (ab+bc)\cos^2 \theta + \I(ad+cd) \cos\theta \sin \theta\\
    (ab+bc)\cos^2 \theta - \I(ad+cd) \cos\theta \sin \theta &
    (b^2+c^2) \cos^2 \theta + d^2 \sin^2 \theta
    \end{bmatrix}.
\end{align*}
Elementary calculations show that, the eigenvalues of the matrix above are
\begin{align*}
\bhalf \Big ( 
(a^2 + 2b^2+c^2) \cos^2\theta &+ 2 d^2 \sin^2\theta\\
&\pm \sqrt{(a^2-c^2)^2 \cos^4 \theta + 4(a+c)^2 \cos^2\theta (b^2 \cos^2\theta + d^2 \sin^2\theta)} \Big ).
\end{align*}
%Since $(a^2 + 2b^2+c^2) \cos^2\theta + 2 d^2 \sin^2\theta \geq 0$, 
The maximum eigenvalue is the one with a $+$ sign in front of the square root term.
\end{proof}
%Using Lemma~\ref{lem:proof-lem-numrad-sub1}, we will show that the supremum of
From \eqref{eq:proof-lem-numrad-2}, one can notice that the norm is an even and periodic function (with period $\pi$) of $\theta$, so it suffices to take the supremum over $\theta \in [0, \pi/2]$. 
We can then use the change of variables $\cos^2 \theta = x$ and $\sin^2 \theta = 1-x$ to write the RHS of \eqref{eq:proof-lem-numrad-2} as 
\begin{equation}
\label{eq:proof-lem-numrad-3}
    2d^2 + (a^2+2b^2+c^2-2d^2) x + 
    \sqrt{((a^2-c^2)^2 + 4(a+c)^2 (b^2-d^2)) x^2 + 
    4(a+c)^2 d^2 x}.
\end{equation}
With this new variable $x$, it suffices to show that 
\begin{equation}
\label{eq:proof-lem-numrad-6}
    a^2+2b^2+c^2-2d^2 \geq 0 ~~\text{and}~~
    (a^2-c^2)^2 + 4(a+c)^2 (b^2-d^2) \geq 0,
\end{equation}
because this implies that \eqref{eq:proof-lem-numrad-3} is an increasing function of $x$ and the supremum over $x \in [0,1]$ is attained at $x = 1$, as desired.

Taking a closer look at \eqref{eq:proof-lem-numrad-6}, one can realize that
\begin{align*}
    a^2+2b^2+c^2-2d^2 &= \frac{(a^2-c^2)^2+4(a+c)^2(b^2-d^2)}{(a+c)^2} + 2(ac-b^2+d^2)\\
    &=
    (a-c)^2 +4(b^2-d^2) + 2(ac-b^2+d^2).
\end{align*}
This means showing
\begin{align}
\label{eq:proof-lem-numrad-7}
    ac-b^2+d^2 \geq 0~~\text{and}~~(a-c)^2 +4(b^2-d^2) \geq 0
\end{align}
implies the two inequalities in \eqref{eq:proof-lem-numrad-6}.
Substituting the entries in \eqref{eq:proof-lem-numrad-4} and \eqref{eq:proof-lem-numrad-5} to $a$, $b$, $c$, and $d$ in  \eqref{eq:proof-lem-numrad-7} and rearranging terms, we get
\begin{align*}
    &ac - b^2 + d^2 = \frac{4 \alpha \beta (ps-qr)^2}{4(ps-qr)^2} = \alpha \beta \geq 0,\\
    &(a-c)^2 + 4(b^2-d^2) = \frac{4 (\alpha- \beta)^2 (ps-qr)^2}{4(ps-qr)^2} = (\alpha- \beta)^2 \geq 0.
\end{align*}

Consequently, for $w(\mQ \mLambda \mQ^{-1})$, the supremum
\begin{equation*}
    \sup_{\theta \in \reals} \norm{\cos \theta \left(\frac{\mQ \mLambda \mQ^{-1} + \mQ^{-T} \mLambda \mQ^{T}}{2}\right) + \I \sin \theta \left(\frac{\mQ \mLambda \mQ^{-1} - \mQ^{-T} \mLambda \mQ^{T}}{2}\right)}
\end{equation*}
is attained at $\theta = 0$, proving that
\begin{equation*}
    w(\mQ \mLambda \mQ^{-1}) = \norm{\frac{\mQ \mLambda \mQ^{-1} + \mQ^{-T} \mLambda \mQ^{T}}{2}}.
\end{equation*}

\subsection{Proof of Lemma~\ref{lem:amgm}}
\label{sec:proof-lem-amgm}

Define the data matrix $\mX \defeq \begin{bmatrix} \vx_1 & \cdots & \vx_n \end{bmatrix}$. Then, $\gdmean$ can be written as
\begin{equation}
\label{eq:proof-lem-amgm-3}
    \gdmean = \left(\mI - \frac{\eta}{n} \mX \mX^T \right)^n
    = \sum_{m=0}^n \binom{n}{m} \left (-\frac{\eta}{n} \right )^m (\mX \mX^T)^m.
\end{equation}
Next, we will write $\rsmean$ in a similar way as the proof of Lemma~\ref{lem:ssprobdep} (see Appendix~\ref{sec:proof-lem-ssprobdep}).
For the given unit vectors $\vx_1, \dots, \vx_n \in \reals^d$ and any permutation $\sigma \in \mc S_n$, we define the following matrices for $m = 0 , \dots, n$, which correspond to the matrix polynomials defined in~\eqref{eq:elem-sym-poly}:
\begin{align*}
    e_m(\sigma) \defeq \sum_{1 \leq i_1 < i_2 < \cdots < i_m \leq n} \vx_{\sigma(i_1)} \vx_{\sigma(i_1)}^T  \vx_{\sigma(i_2)} \vx_{\sigma(i_2)}^T \cdots \vx_{\sigma(i_m)} \vx_{\sigma(i_m)}^T.
\end{align*}
Notice that $e_0(\sigma) = \mI$ and $e_1(\sigma) = \sum_{i=1}^n \vx_i \vx_i^T = \mX \mX^T$ are independent of the permutation $\sigma$.
The mean of $e_m(\sigma)$ over all permutations $\sigma \in \mc S_n$ is
\begin{align*}
    \overline{e_m} \defeq \frac{1}{n!}\sum_{\sigma \in \mc S_n} e_m(\sigma)
    = \frac{1}{m!} \sum_{\substack{i_1, i_2, \dots, i_m\\\text{all distinct}}} \vx_{i_1} \vx_{i_1}^T  \vx_{i_2} \vx_{i_2}^T \cdots \vx_{i_m} \vx_{i_m}^T.
\end{align*}
Using these matrices, we can write $\rsmean$ as
\begin{equation}
\label{eq:proof-lem-amgm-4}
    \rsmean
    = \sum_{m=0}^n (-\eta)^m \overline{e_m}
    = \mI - \eta \mX\mX^T + \sum_{m=2}^n (-\eta)^m \overline{e_m}.
\end{equation}
From \eqref{eq:proof-lem-amgm-3} and \eqref{eq:proof-lem-amgm-4}, notice that the degree-0 and degree-1 (in $\eta$) terms in $\gdmean$ and $\rsmean$ cancel each other:
\begin{align}
    \gdmean - \rsmean
    &= \sum_{m=0}^n \binom{n}{m} \left (-\frac{\eta}{n} \right )^m (\mX \mX^T)^m
    - \sum_{m=0}^n (-\eta)^m \overline{e_m}\notag\\
    &= \sum_{m=2}^n (-\eta)^m \left ( \binom{n}{m} \frac{(\mX\mX^T)^m}{n^m} - \overline{e_m} \right ).
    \label{eq:proof-lem-amgm-6}
\end{align}
The rest of the proof goes as follows. We will bound the norm of the degree-2 and degree-3 terms in $\gdmean - \rsmean$~\eqref{eq:proof-lem-amgm-6}. We will then bound the higher-degree terms separately.

We first consider the degree-2 terms and bound their minimum eigenvalue:
\begin{align}
    \eta^2 \left ( \binom{n}{2} \frac{(\mX\mX^T)^2}{n^2} - \overline{e_2} \right )
    &=
    \frac{\eta^2}{2} \left ( \frac{n-1}{n} \left (\sum_{i=1}^n \vx_i\vx_i^T \right )^2 - \sum_{\substack{i,j \in [n],i \neq j}} \vx_i \vx_i^T \vx_j \vx_j^T \right )\notag\\
    &=
    \frac{\eta^2}{2n} \left ( (n-1) \sum_{i=1}^n \vx_i \vx_i^T - \sum_{\substack{i,j \in [n],i \neq j}} \vx_i \vx_i^T \vx_j \vx_j^T \right ).
    \label{eq:proof-lem-amgm-1}
\end{align}
Now, for any unit vector $\vu \in \reals^d$, we have
\begin{align}
    &\vu^T\left ( (n-1) \sum_{i=1}^n \vx_i \vx_i^T - \sum_{\substack{i,j \in [n],i \neq j}} \vx_i \vx_i^T \vx_j \vx_j^T \right )\vu\notag\\
    =&
    (n-1)\sum_{i=1}^n (\vu^T \vx_i)^2 - \sum_{i,j \in [n],i \neq j} (\vx_i^T \vx_j) (\vu^T \vx_i) (\vu^T \vx_j)\notag\\
    \geq&
    (n-1)\sum_{i=1}^n (\vu^T \vx_i)^2
    - \delta \sum_{i,j \in [n],i \neq j} |\vu^T \vx_i| |\vx^T \vx_j|\notag\\
    =& (n-1)(1-\delta) \sum_{i=1}^n (\vu^T \vx_i)^2 + \frac{\delta}{2}\sum_{i,j \in [n],i \neq j} (|\vu^T \vx_i| -  |\vx^T \vx_j|)^2\notag\\
    \geq& (n-1)(1-\delta)s_{\min}.
    \label{eq:proof-lem-amgm-2}
\end{align}
Combining \eqref{eq:proof-lem-amgm-1} and \eqref{eq:proof-lem-amgm-2}, we bound the minimum eigenvalue of the degree-2 terms in $\gdmean-\rsmean$:
\begin{equation}
\label{eq:proof-lem-amgm-7}
    \eta^2 \left ( \binom{n}{2} \frac{(\mX\mX^T)^2}{n^2} - \overline{e_2} \right )
    \succeq \frac{\eta^2 (n-1)(1-\delta) s_{\min}}{2n} \mI.
\end{equation}

Next, we bound the spectral norm of the degree-3 terms.
\begin{align}
    &\eta^3 \left ( \binom{n}{3} \frac{(\mX\mX^T)^3}{n^3} - \overline{e_3} \right )\notag\\
    =&
    \frac{\eta^3}{6} \Bigg ( \frac{(n-1)(n-2)}{n^2} \left (\sum_{i=1}^n \vx_i\vx_i^T \right )^3 - \sum_{\substack{i,j,k \in [n]\\i,j,k ~\text{distinct}}} \vx_i \vx_i^T \vx_j \vx_j^T \vx_k \vx_k^T \Bigg )\notag\\
    =&
    \frac{\eta^3}{6} \Bigg ( 
    \left (\frac{(n-1)(n-2)}{n^2} - 1 \right ) \left (\sum_{i=1}^n \vx_i\vx_i^T \right )^3 
    + \left (\sum_{i=1}^n \vx_i\vx_i^T \right )^3
    - \sum_{\substack{i,j,k \in [n]\\i,j,k ~\text{distinct}}} \vx_i \vx_i^T \vx_j \vx_j^T \vx_k \vx_k^T \Bigg ).\notag
    % \\
    % =&
    % \frac{\eta^3}{6} \Bigg ( 
    % -\frac{3n-2}{n^2} (\mX\mX^T)^3 
    % + \sum_{i=1}^n \vx_i \vx_i^T \Bigg ( \sum_{j=1}^n \vx_j \vx_j^T \Bigg) \vx_i \vx_i^T\notag\\
    % &\qquad\qquad+\sum_{\substack{i,j \in [n], i\neq j}} \vx_i \vx_i^T (\vx_i \vx_i^T + \vx_j \vx_j^T) \vx_j \vx_j^T \Bigg )\notag\\
    %\label{eq:proof-lem-amgm-1}
\end{align}
Bounding the first part separately, we have
\begin{align*}
    \norm{\frac{\eta^3}{6}
    \left (\frac{(n-1)(n-2)}{n^2} - 1 \right ) \left (\sum_{i=1}^n \vx_i\vx_i^T \right )^3}
    \leq \frac{\eta^3 (3n-2)}{6n^2} s_{\max}^3
    \leq \frac{\eta^3 s_{\max}^3}{2n}.
\end{align*}
and the second part can be written as
\begin{align*}
    &\frac{\eta^3}{6}
    \Bigg ( \left (\sum_{i=1}^n \vx_i\vx_i^T \right )^3
    - \sum_{\substack{i,j,k \in [n]\\i,j,k ~\text{distinct}}} \vx_i \vx_i^T \vx_j \vx_j^T \vx_k \vx_k^T \Bigg )\\
    =&
    \frac{\eta^3}{6} \Bigg (\sum_{i=1}^n \vx_i \vx_i^T \Bigg ( \sum_{j=1}^n \vx_j \vx_j^T \Bigg) \vx_i \vx_i^T\notag
    +\sum_{\substack{i,j \in [n], i\neq j}} \vx_i \vx_i^T (\vx_i \vx_i^T + \vx_j \vx_j^T) \vx_j \vx_j^T \Bigg )\\
    =&\frac{\eta^3}{6}\mX \mS \mX^T,
\end{align*}
where the matrix $\mS \in \reals^{n \times n}$ in between is defined as
\begin{align*}
    [\mS]_{i,j} \defeq 
    \begin{cases}
    \vx_i^T \left ( \sum_{k=1}^n \vx_k \vx_k^T \right ) \vx_i = 1 + \sum_{k\in[n]\setminus \{i\}} (\vx_k^T \vx_i)^2 & \text{ if } i=j,\\
    \vx_i^T (\vx_i \vx_i^T + \vx_j \vx_j^T) \vx_j
    = 2\vx_i^T\vx_j & \text { if } i \neq j.
    \end{cases}
\end{align*}
The spectral norm $\norm{\mS}$ of $\mS$ is bounded above by its Frobenius norm, so
\begin{equation*}
    \norm{\mS}^2 \leq \lfro{\mS}^2 \leq
    n (1+(n-1)\delta^2)^2 + n(n-1) (2\delta)^2
    \leq n(n^2\delta^4+6n\delta^2+1),
\end{equation*}
which leads to
\begin{equation*}
\norm{\frac{\eta^3}{6}\mX \mS \mX^T}
\leq \frac{\eta^3}{6}\norm{\mX}^2 \norm{\mS}
=\frac{\eta^3}{6}\norm{\mX\mX^T} \norm{\mS}
\leq \frac{\eta^3}{6}s_{\max} n^{1/2}(n^2\delta^4+6n\delta^2+1)^{1/2}.
\end{equation*}
In summary, the norm of the degree-3 terms is bounded by
\begin{equation}
\label{eq:proof-lem-amgm-8}
    \norm{\eta^3 \left ( \binom{n}{3} \frac{(\mX\mX^T)^3}{n^3} - \overline{e_3} \right )}
    \leq
    \frac{\eta^3}{6}
    \left (
    \frac{3s_{\max}^3}{n} 
    +
    s_{\max} n^{1/2}(n^2\delta^4+6n\delta^2+1)^{1/2}
    \right ).
\end{equation}

Next, we bound the norm of higher-degree terms in $\gdmean$ and $\rsmean$ separately. For $\gdmean$,
\begin{align}
\label{eq:proof-lem-amgm-9}
    \norm{\sum_{m=4}^n \binom{n}{m} \left (-\frac{\eta}{n} \right )^m (\mX \mX^T)^m}
    \leq \sum_{m=4}^n  \binom{n}{m} \left (\frac{\eta}{n} \right )^m \norm{\mX\mX^T}^m    \leq \sum_{m=4}^n \eta^m s_{\max}^m.
\end{align}
For $\rsmean$, the matrix $\overline{e_m}$ can be written as the following form (this argument applies for any $m \geq 2$):
\begin{align*}
    \overline{e_m} &= \frac{1}{m!} \sum_{\substack{i_1, i_2, \dots, i_m\\\text{all distinct}}} \vx_{i_1} \vx_{i_1}^T  \vx_{i_2} \vx_{i_2}^T \cdots \vx_{i_m} \vx_{i_m}^T
    =\frac{1}{m!} \mX \mT_m \mX^T,
\end{align*}
where $[\mT_m]_{i,i} = 0$ for all $i \in [n]$, and for any $i \neq j$,
\begin{align*}
    [\mT_m]_{i,j} = \sum_{\substack{i_1, \dots, i_{m-2} \in [n] \setminus \{i,j\}\\i_1, \dots, i_{m-2} \text{ all distinct }}} \vx_i^T \vx_{i_1} \vx_{i_1}^T \vx_{i_2} \vx_{i_2}^T \cdots \vx_{i_{m-2}} \vx_{i_{m-2}}^T \vx_j.
\end{align*}
Since there are $\frac{(n-2)!}{(n-m)!}$ terms in the sum and each term is a product of $m-1$ inner products between distinct vectors, each of the entries satisfies $|[\mT_m]_{i,j}| \leq \frac{(n-2)!}{(n-m)!} \delta^{m-1}$. Therefore, using the fact that the spectral norm $\norm{\mT_m}$ of $\mT_m$ is bounded above by its Frobenius norm $\lfro{\mT_m}$, we have
\begin{align}
\label{eq:proof-lem-amgm-5}
    \norm{\mT_m}^2 \leq \lfro{\mT_m}^2 \leq n(n-1)\cdot \left ( \frac{(n-2)!}{(n-m)!} \delta^{m-1} \right )^2
    \leq n^{2m-2} \delta^{2m-2}.
\end{align}
Using \eqref{eq:proof-lem-amgm-5}, we can bound $\norm{\overline{e_m}}$:
\begin{align}
\label{eq:proof-lem-amgm-10}
    \norm{\overline{e_m}} = \frac{1}{m!} \norm{\mX \mT_m \mX^T} 
    \leq \frac{1}{m!} \norm{\mX}^2 \norm{\mT_m} 
    = \frac{1}{m!} \norm{\mX\mX^T} \norm{\mT_m} 
    \leq s_{\max} n^{m-1} \delta^{m-1}.
\end{align}
Combining \eqref{eq:proof-lem-amgm-7}, \eqref{eq:proof-lem-amgm-8}, \eqref{eq:proof-lem-amgm-9}, and \eqref{eq:proof-lem-amgm-10}, one get the lower bound~\eqref{eq:proof-sketch-2} on the minimum eigenvalue $\lambda_{\min}(\gdmean-\rsmean)$ of the difference.
The other lower bound~\eqref{eq:proof-sketch-3} on $\lambda_{\min}(\rsmean)$ follows from the definition of $\rsmean$~\eqref{eq:proof-lem-amgm-4} and the upper bounds on $\norm{\overline{e_m}}$~\eqref{eq:proof-lem-amgm-10}.

Finally, to prove the last statement that the RHSs of \eqref{eq:proof-sketch-2} and \eqref{eq:proof-sketch-3} are nonnegative, we substitute $\delta = n^{-1/2}$, $s_{\min} = n^{-1/4}$, and $s_{\max} = n^{1/4}$ from the assumptions to the RHSs of \eqref{eq:proof-sketch-2} and \eqref{eq:proof-sketch-3}. 
For \eqref{eq:proof-sketch-2}, we get
\begin{align}
\label{eq:proof-sketch-4}
    &\frac{\eta^2 (n-1)(1-n^{-1/2})}{2n^{5/4}} - \sum_{m=4}^n \eta^m(n^{m/4} + n^{1/4} n^{(m-1)/2}) -\frac{\eta^3}{6} \left (3n^{-1/4} + \sqrt 8 n^{3/4} \right )\notag
    \\
    \geq &\frac{3 \eta^2}{16n^{1/4}} -
    2 n^{-1/4} \sum_{m=4}^n \eta^m n^{m/2}
    - \frac{2\eta^3 n^{3/4}}{3}
    \geq \frac{3 \eta^2}{16n^{1/4}} - \frac{2\eta^3 n^{3/4}}{3} - \frac{2 \eta^4 n^{7/4}}{1-\eta n^{1/2}},
\end{align}
where we used $\frac{(n-1)(1-n^{-1/2})}{2n} \geq \frac{3}{16}$ for $n \geq 4$, $n^{m/4} \leq n^{1/4} n^{(m-1)/2}$ for $m,n \geq 4$, and $3n^{-1/4} \leq \frac{3n^{3/4}}{4}$ for $n \geq 4$.
Elementary calculations prove that the RHS of \eqref{eq:proof-sketch-4} is nonnegative for $\eta \in [0, \frac{1}{6n}]$.

Similarly, for the RHS of \eqref{eq:proof-sketch-3}, we have
\begin{align}
\label{eq:proof-sketch-5}
    1 - n^{-1/4} \sum_{m=1}^n \eta^m n^{m/2}
    \geq 1 - \frac{\eta n^{1/4}}{1-\eta n^{1/2}}.
\end{align}
Again, the RHS of \eqref{eq:proof-sketch-5} is nonnegative for $\eta \in [0, \frac{1}{6n}]$.